%% file: main.tex
\documentclass[nohyperref]{article}
\usepackage[margin=1in]{geometry}

\usepackage{natbib}
\usepackage{microtype}
\usepackage{graphicx}
\usepackage{subfigure}
\usepackage{booktabs}
\usepackage{balance}
\usepackage{hyperref}

\usepackage[parfill]{parskip}
\usepackage[utf8]{inputenc} 
\usepackage[T1]{fontenc}

\usepackage{macro_math}
\usepackage{setspace}
\usepackage{adjustbox}
\usepackage{wrapfig}
\usepackage{float}
\usepackage{colortbl}
\usepackage{multirow}

\usepackage{hyperref}
\hypersetup{colorlinks,linkcolor={blue},citecolor={blue},urlcolor={red}}

\newcommand{\cmnist}{CMNIST$^*$}

\definecolor{Gray}{gray}{0.85}
\definecolor{LightCyan}{rgb}{0.7,1,1}
\newcolumntype{a}{>{\columncolor{Gray}}c}
\newcolumntype{b}{>{\columncolor{LightCyan}}c}

\usepackage{amsmath}
\usepackage{amssymb}
\usepackage{mathtools}
\usepackage{amsthm}

\usepackage[capitalize,noabbrev]{cleveref}

\newcommand{\alig}{\text{align}}

\begin{document}

\title{Correct-\textsc{n}-Contrast: A Contrastive Approach for Improving Robustness to Spurious Correlations}
\author{Michael Zhang$^\dagger$\thanks{Corresponding author}, Nimit S. Sohoni$^\dagger$, Hongyang R. Zhang$^\ddagger$, Chelsea Finn$^\dagger$ \& Christopher R\'{e}$^\dagger$  \\
$^\dagger$Stanford University,
$^\ddagger$Northeastern University \\
\texttt{\{mzhang,nims,hongyang,cbfinn,chrismre\}@cs.stanford.edu} \\
}

\maketitle

\begin{abstract}
\noindent
Spurious correlations pose a major challenge for robust machine learning. Models trained with empirical risk minimization (ERM) may learn to rely on correlations between class labels and spurious attributes, leading to poor performance on data groups without these correlations. 
This is particularly challenging to address when spurious attribute labels are unavailable.
To improve worst-group performance on spuriously correlated data without training attribute labels, we propose Correct-\textsc{n}-Contrast (\name{}), a contrastive approach to directly learn representations robust to spurious correlations. As ERM models can be good spurious attribute predictors, \name{} works by (1) using a trained ERM model’s outputs to identify samples with the same class but dissimilar spurious features, and (2) training a robust model with contrastive learning to learn similar representations for same-class samples. To support \name{}, we introduce new connections between worst-group error and a representation \textit{alignment loss} that \name{} aims to minimize. We empirically observe that worst-group error closely tracks with alignment loss, and prove that the alignment loss over a class helps upper-bound the class's worst-group vs. average error gap. On popular benchmarks, \name{} reduces alignment loss drastically, and achieves state-of-the-art worst-group accuracy by $\textbf{3.6}\%$ average absolute lift. \name{} is also competitive with oracle methods that require group labels.

\end{abstract}

\begin{figure*}[t]
  \centering
  \includegraphics[width=1\textwidth]{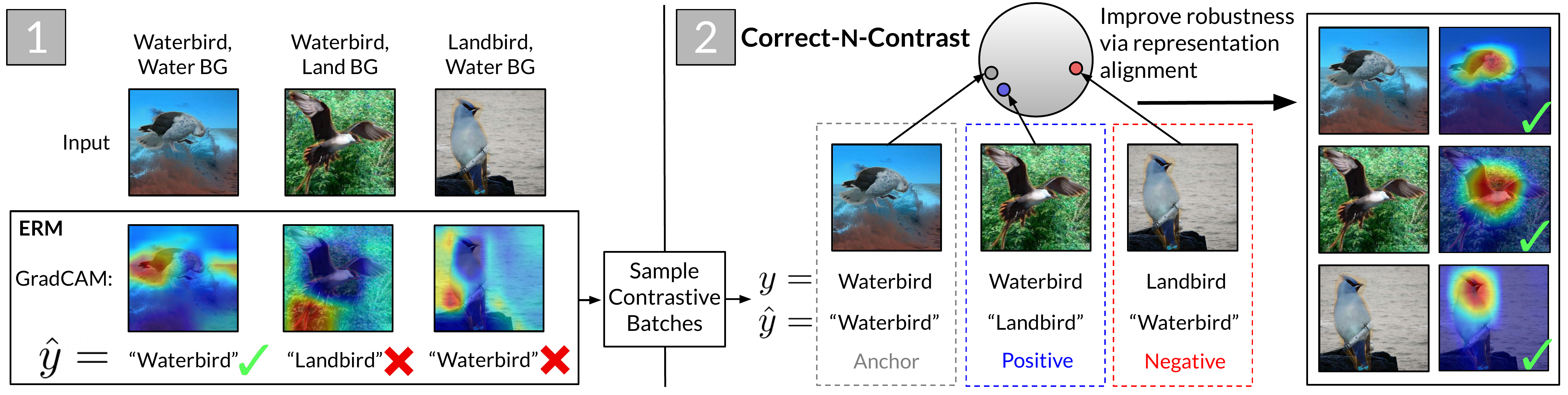}

  \caption{{(1) ERM-trained models classify by spurious features, shown via GradCAM \citep{selvaraju2017grad}. (2) \name{} learns similar representations for same-class samples with different ERM predictions to ignore spurious attributes and classify samples correctly.}}
  \label{fig:pull_gradcam_embeddings}
\end{figure*}

\section{Introduction}
\label{intro}

For many tasks, deep neural networks (NNs) are negatively affected by spurious correlations---dependencies between observed features and class labels that only hold for certain groups of data. For example, consider classifying cows versus camels in natural images. 90\% of cows may appear on grass and 90\% of camels on sand. A model trained with standard ERM may learn to minimize average training error by relying on the spurious \textit{background} attribute instead of the desired \textit{animal}, performing well on average yet obtaining high \textit{worst-group} test error (e.g. by misclassifying cows on sand as camels) \citep{ribeiro2016should, beery2018recognition}.
This illustrates a fundamental issue where trained models can become systematically biased due to spurious attributes in their data, which is critical to address in fields ranging from algorithmic fairness to medical machine learning \citep{blodgett2016demographic, buolamwini2018gender, hashimoto2018fairness}.

How can we improve robustness to spurious correlations and reduce worst-group error? If group or spurious attribute labels are available, one option is to reweight groups during training to directly minimize worst-group error, e.g. via group DRO (GDRO) \cite{sagawa2019distributionally}. 
However, such group labels may be expensive to obtain or unknown \emph{a priori} \citep{oakden2020hidden}, limiting the applicability of these ``oracle'' methods. 
To tackle spurious correlations without these labels, many prior works recognize that because ERM models can learn spurious correlations, such models can help infer the groups or spurious attributes. These methods thus first infer groups with a trained ERM model, before using these inferred groups to train a more robust model. 
For example, \citet{NEURIPS2020_e0688d13} cluster an ERM model's feature representations to infer groups, before running GDRO using these clusters as groups. \citet{NEURIPS2020_eddc3427, liu2021just} treat an ERM model's misclassifications as minority group samples, and upweight these groups to train a robust model. \citet{Ahmed2021SystematicGW, creager2021environment} use invariance objectives to both infer groups with an ERM model and train a more robust model.
However, while these methods significantly improve worst-group error over ERM without training group labels, they still perform worse than methods that use group labels (e.g., GDRO).

In this work, we thus aim to further improve worst-group accuracy without requiring group labels. We start with the motivating observation that across spurious correlation settings, a neural network's worst-group accuracy strongly tracks how well its representations---i.e., the outputs from its last hidden layer---exhibit dependence \textit{only} on ground-truth labels, and \textit{not} on spurious attributes. We quantify this property via estimated mutual information as well as a notion of geometric representation \textit{alignment}, inspired by prior work in contrastive learning \citep{wang2020understanding}. Here, the alignment measures how close samples with the same class but different spurious attributes embed in representation space. We first empirically observe that this dependence consistently holds across datasets with increasingly strong spurious correlations, and also helps explain when upweighting methods (e.g., \textsc{Jtt} \citep{liu2021just}) improve worst-group error over ERM. We then theoretically show that a model's alignment loss for a class (i.e., the average representation distance between same-class samples with different spurious attributes) can help upper bound the worst-group versus average error gap for that class. Thus, by improving alignment while keeping average class error low, we can improve the class's worst-group error. However, current methods do not directly optimize for these representation-level properties, suggesting one underexplored direction to improve  worst-group error.

We thus propose Correct-\textsc{n}-Contrast (\name{}), a two-stage contrastive learning approach for better-aligned representations and improved robustness to spurious correlations. The key idea is to use contrastive learning to ``push together'' representations for samples with the same class but different spurious attributes, while ``pulling apart'' those with different classes and the same spurious attribute (Fig.~\ref{fig:pull_gradcam_embeddings}). \name{} thus improves intra-class alignment while also maintaining inter-class separability, encouraging models to rely on features predictive of class labels but not spurious attributes. As we do not know the spurious attribute labels, \name{} first infers the spurious attributes using a regularized ERM model trained to predict class labels (similar to prior work).
Different from these works, we use these predictions to train another robust model via contrastive learning and a novel sampling strategy.
We randomly sample an anchor, select samples with the same class but different ERM predictions as hard positives {to ``push together''}, and samples from different classes but the same ERM prediction as hard negatives {to ``pull apart''}. This encourages ignoring spurious differences and learning class-specific similarities between anchor and positives, and also conversely encourages ignoring spurious similarities and learning class-specific differences between anchor and negatives.
\name{} thus \underline{correct}s for the ERM model's learned spurious correlations via \underline{contrast}ive learning.

We evaluate \name{} on four  spurious correlation benchmarks. Among methods that do not assume training group labels, \name{} substantially improves worst-group accuracy, obtaining up to \textbf{7.7\%} absolute lift (from 81.1\% to 88.8\% on CelebA), and averaging \textbf{3.6\%} absolute lift over the second-best method averaged over all tasks (\textsc{Jtt}, \citet{liu2021just}).
\name{} also nearly closes the worst-group accuracy gap with methods requiring training group labels, only falling short of GDRO's worst-group accuracy by 0.9 points on average. To help explain this lift, we find that \name{} indeed improves alignment and learns representations with substantially higher dependence on classes over spurious attributes. 
Finally, we run additional ablation experiments that show that: \name{} is more robust to noisy ERM predictions than prior methods; with spurious attribute labels, \name{} improves worst-group accuracy over GDRO by \textbf{0.9\%} absolute on average; and \name{}'s sampling approach improves performance compared to alternative approaches. These results show that our contrastive learning and sampling strategies are effective techniques to tackle spurious correlations.

\textbf{Summary of contributions}. We summarize our contributions as follows:
\begin{enumerate}
    \item We empirically show that a model's worst-group performance correlates with the model's alignment loss between different groups within a class, and theoretically analyze this connection.
    \item We propose \name{}, a two-stage contrastive approach with a hard negative sampling scheme to improve representation alignment and thus train models more robust to spurious correlations.
    \item We show that \name{} achieves state-of-the-art worst-group accuracy on three benchmarks and learns better-aligned representations less reliant on spurious features.
\end{enumerate}

\section{Preliminaries}
\label{sec:preliminaries}

\textbf{Problem setup.}
We present our setting and the loss objectives following \citet{sagawa2019distributionally}.
Let $X = \set{x_1,\ldots,x_n}$ and $Y = \set{y_1,\ldots,y_n}$ be our training dataset of size $n$.
Each datapoint has an observed feature vector $x_i \in \mathcal{X}$, label $y_i \in \mathcal{Y}$, and \emph{unobserved} spurious attribute $a_i \in \mathcal{A}$. The set of groups $\mathcal{G}$ is defined as the set of all combinations of class label and spurious attribute pairs, i.e. $\mathcal{G} = \mathcal{Y} \times \mathcal{A}$.
Let $C = |\cY|$ be the number of classes and $K = |\mathcal{G}|$ be the number of groups.
We assume that each example $(x_i, y_i, a_i)$ is drawn from an unknown joint distribution $P$, and at least one sample from each group is observed in the training data.
Let $P_g$ be the distribution conditioned on $(y, a) = g$, for any $g\in\cG$.
Given a model $f_\theta : \mathcal{X} \mapsto \real^C$ and a convex loss $\mathcal{\ell}: \cX\times\cY\mapsto\real$, the \emph{worst-group} loss is:
{\begin{equation}
    \mathcal{L}_\text{wg}(f_\theta) \define \max_{g \in \cG}\; \mathbb{E}_{(x, y, a) \sim P_g} [\ell(f_\theta(x), y)].
\label{eq:worst_group_error}
\end{equation}}
ERM minimizes the training loss as a surrogate for the expected average population loss $\cL_{\text{avg}}$:
{\begin{equation}
    \mathcal{L}_\text{avg}(f_\theta) \define \mathbb{E}_{(x, y, a) \sim P} [\ell(f_\theta(x), y)] 
    \label{eq:erm_loss}
\end{equation}}
While ERM is the standard way to train NNs, spurious correlations can cause ERM to obtain high minority group error even with low average error. Minimizing the empirical version of \eqref{eq:worst_group_error} via GDRO is a strong baseline for improving worst-group error, if training group labels $\{a_1,\dots,a_n\}$ are available \citep{sagawa2019distributionally}. We tackle the more challenging setting where training group labels are \textit{not available}.

\textbf{Contrastive learning.}
We briefly describe contrastive learning \citep{chen2020simple}, a central component of our approach.
Let $f_\theta$ be a neural network model with parameters $\theta$.
Let the encoder $f_\text{enc} : \mathcal{X} \mapsto \real^d$ be the feature representation layers of $f_{\theta}$.
Let $f_\text{cls}: \real^d \mapsto \real^{C}$ be the  classification layer of $f_{\theta}$, which maps encoder representations to one-hot label vectors. We learn $f_\text{enc}$ with 
the \textit{supervised contrastive loss} $\mathcal{L}_\text{con}^\text{sup}$ proposed in \citet{NEURIPS2020_d89a66c7}. For each anchor $x$, we sample $M$ positives $\{x^+_m\}_{m = 1}^M$ and $N$ negatives $\{x^-_n\}_{n = 1}^N$. Let $y, \{y_m^+\}_{m=1}^M, \{y_n^-\}_{n=1}^N$ be the labels and $z, \{z_m^+\}_{m=1}^M, \{z_n^-\}_{n=1}^N$ be the normalized outputs of $f_\text{enc}(x)$ for the anchor, positives, and negatives respectively. With input $x$ mapped to $z$, the training objective for the encoder is to minimize $\mathcal{L}_\text{con}^\text{sup}(x; f_\text{enc})$, defined as
\begin{equation}
\mathcal{L}_\text{con}^\text{sup}(x; f_\text{enc}) = 
    \exarg{z, \set{z_m^+}_{m=1}^M, \set{z_n^-}_{n=1}^N}{ - \log \frac{\exp( z^{\top} z_m^+ / \tau)}{\sum_{m=1}^M \exp( z^{\top} z_m^+ / \tau) + \sum_{n=1}^N \exp( z^{\top} z_n^- / \tau)} }
    \label{eq:supcon_easy_anchor}
\end{equation}
where $\tau > 0$ is a scalar temperature hyperparameter. 
Minimizing Eq.~\ref{eq:supcon_easy_anchor} leads to $z$ being closer to $z^+$ than $z^-$ in representation space. 

\section{Spurious correlations' impact on learned data representations}
\label{sec:erm_observations}
We present our key observation that a model's worst-group accuracy correlates with how well its learned representations depends on the class labels and \emph{not} the spurious attributes. We draw  connections between a neural network's worst-group error and its representation metrics, noting a strong inverse relationship between worst-group accuracy and a class-specific alignment loss. We then theoretically justify this relationship.

\subsection{Understanding worst-group performance using representation metrics}\label{sec_motivate_erm}

We first illustrate that when neural networks are trained with standard ERM on spuriously correlated data, their hidden layer representations exhibit high dependence on the spurious attribute. To better understand and connect this behavior to worst-group error, we quantify these results using representation alignment (cf. Eq.~\ref{eq:loss_align}) and mutual information metrics. We observe that these metrics explain trends in ERM's worst-group accuracy on various spuriously correlated datasets. These trends also apply to upsampling methods that mitigate the impact of spurious features.

\begin{wrapfigure}{r}{0.52\textwidth}
\vspace{-0.65cm}
\includegraphics[width=1\linewidth]{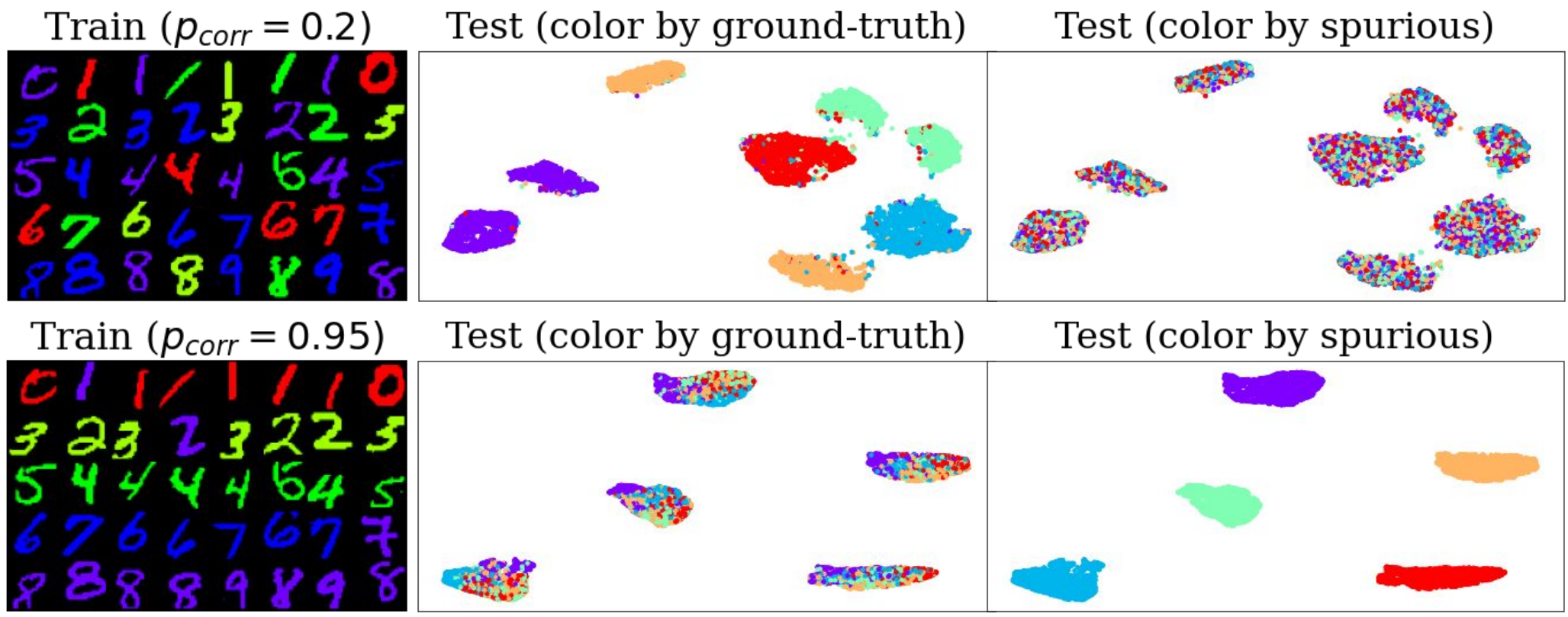}
\vspace{-0.65cm}
\caption{{ERM-learned representation UMAPs, trained on spuriously uncorrelated ($p_\text{corr} = 0.2$}) vs correlated ($p_\text{corr} = 0.95$) data.}
\label{fig:erm_motivation}
\vspace{-0.35cm}
\end{wrapfigure}

\textbf{Example setup.} We model spurious correlations with \cmnist, a colored MNIST dataset inspired by \citet{arjovsky2019invariant}.
There are $5$ digit classes and $5$ colors. We color a fraction $p_\text{corr}$ of the training samples with a color $a$ associated with each class $y$, and color the test samples uniformly-randomly. To analyze learned representations, we train a LeNet-5 CNN \citep{lecun1989backpropagation} with ERM to predict digit classes, and inspect the outputs of the last hidden layer $z = f_\text{enc}(x)$. As shown in Fig.~\ref{fig:erm_motivation}, with low $p_\text{corr}$, models learn representations with high dependence on the actual digit classes. However, with high $p_\text{corr}$, we learn $z$ highly dependent on $a$, despite only training to predict $y$. 

\textbf{Representation metrics.} To quantify this behavior, 
we use two  metrics designed to capture how well the learned representations exhibit dependence on the class label versus the spurious attributes. First, we compute an \emph{alignment loss} $\hat{\mathcal{L}}_\text{align}(f_{\enc}; g, g')$ between two groups $g = (y, a)$ and $g' = (y, a')$ where $a\neq a'$. This measures how well $f_\text{enc}$ maps samples with the same class, but different spurious attributes, to nearby vectors via Euclidean distance. 

Letting $G$ and $G'$ be subsets of training data in groups $g$ and $g'$ respectively, we define $\hat{\mathcal{L}}_\text{align}(f_\text{enc}; g, g')$ as
\begin{equation}
\hat{\mathcal{L}}_\text{align}(f_\text{enc}; g, g') =
\frac{1}{\abs{G}} \frac{1}{\abs{G'}} \sum_{(x, y, a) \in G}\sum_{(x', y, a') \in G'}{ \bignorm{f_\text{enc}(x) - f_\text{enc}(x')}_2}
    \label{eq:loss_align}
\end{equation}
Thus, lower $\hat{\mathcal{L}}_\text{align}$ means better alignment. We also quantify representation dependence by estimating the mutual information (MI) of a model's learned representations with the class label, i.e. $\hat{I}(Y; Z)$ and the spurious attributes $\hat{I}(A; Z)$. We defer computational details to Appendix~\ref{appendix:experimental_details}.

\begin{figure*}[h]
  \centering
  \includegraphics[width=1\textwidth]{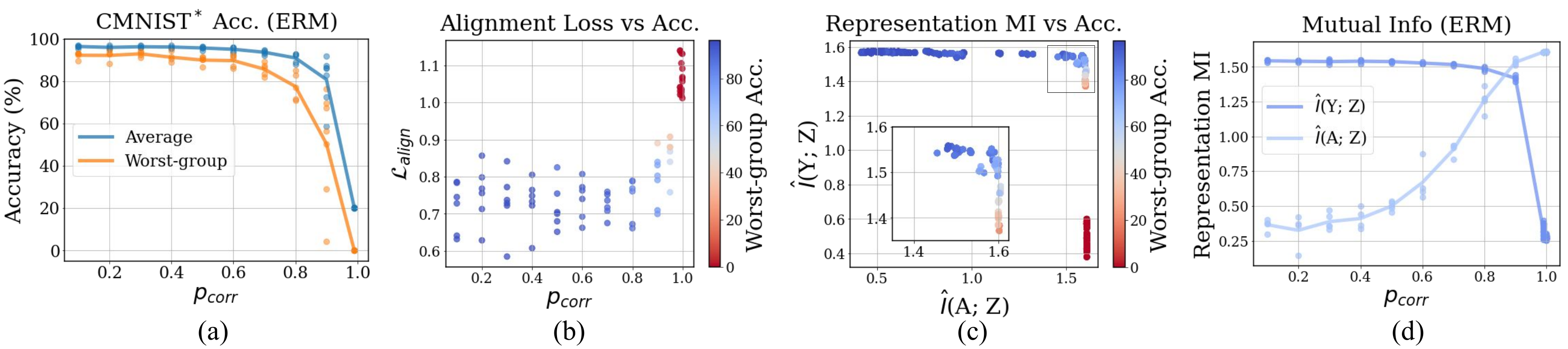}
  \vspace{-0.65cm}
  \caption{{Accuracy and representation metrics from ERM models trained on increasingly spuriously correlated Colored MNIST. Lower worst-group accuracy (Fig. 3a) corresponds to both higher alignment loss (Fig. 3b) and $\hat{I}(Y; Z) < \hat{I}(A; Z)$ (Fig. 3c, Fig. 3d).}}
  \label{fig:erm_ablation_cmnist}
  \vspace{-0.5cm}
\end{figure*}

\begin{wrapfigure}{r}{0.53\textwidth}
\vspace{-0.25cm}
\includegraphics[width=1\linewidth]{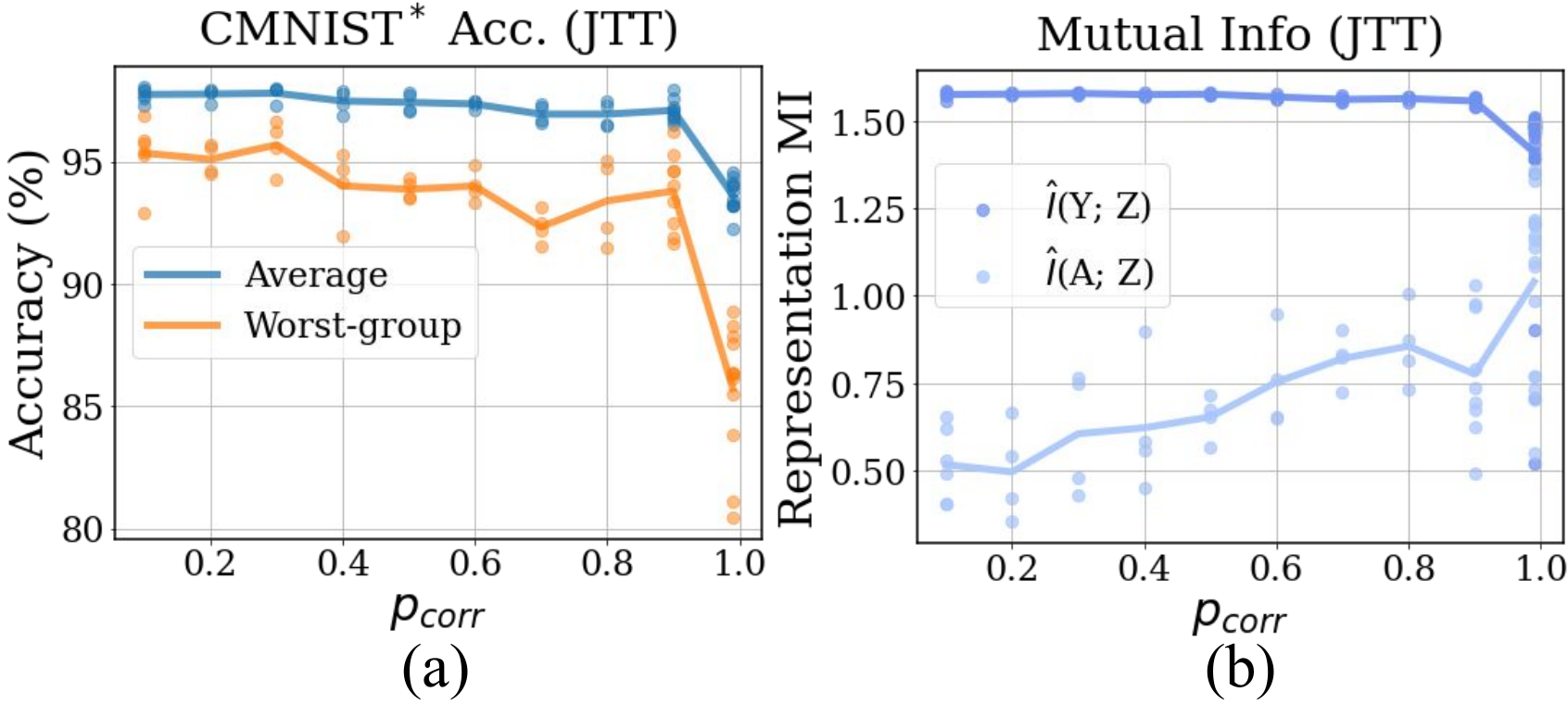}
\vspace{-0.75cm}
\caption{{Higher worst-group accuracy with \textsc{Jtt} (versus Fig.~\ref{fig:erm_ablation_cmnist}a) coincides with keeping $\hat{I}(Y; Z) \gg \hat{I}(A; Z)$.}}
\label{fig:erm_ablation_cmnist_upsampling}
\end{wrapfigure}

\textbf{Results for ERM.} In Fig.~\ref{fig:erm_ablation_cmnist} we find that worst-group error is strongly associated with both alignment and mutual information (MI) metrics. 
As $p_\text{corr}$ increases, ERM models not only drop in worst-group accuracy, but also incur higher alignment loss (Fig.~\ref{fig:erm_ablation_cmnist}ab). Fig.~\ref{fig:erm_ablation_cmnist}c further illustrates this with MI. We plot the estimated MI and worst-group accuracy for models at each epoch. A substantial drop in worst-group accuracy occurs with high $\hat{I}(A; Z)$ (especially when $\hat{I}(A; Z) > \hat{I}(Y; Z)$, even when $\hat{I}(Y; Z)$ is high). Fig.~\ref{fig:erm_ablation_cmnist}d also captures this trend: as $p_{\text{corr}}$ increases, $\hat{I}(A; Z)$ does as well while $\hat{I}(Y; Z)$ decreases.

\textbf{Results for \textsc{Jtt}.} In  Fig.~\ref{fig:erm_ablation_cmnist_upsampling}, we also show that this relation holds when training with another recent (upsampling-based) approach, \textsc{Jtt} \citep{liu2021just}. With high $p_\text{corr}$, \textsc{Jtt} achieves higher worst-group accuracy compared to ERM, and this corresponds to learning representations with high $\hat{I}(Y; Z)$ and low $\hat{I}(A; Z)$. However, we note that \textsc{Jtt} and other previous approaches do not explicitly optimize representation-level metrics, suggesting a new direction to improve worst-group performance.

\input{analysis}

\textbf{Broader  implications.}
We summarize this section with two takeaways: 
\begin{enumerate}
    \item When trained on spuriously correlated datasets, ERM networks learn data representations highly dependent on spurious attributes. Representation clusters \citep{NEURIPS2020_e0688d13} or the ERM model's outputs \citep{liu2021just, NEURIPS2020_eddc3427} can thus serve as (noisy) pseudolabels for spurious attributes.
    \item Both representation metrics correlate with worst-group error, such that a viable way to improve worst-group error is to improve alignment within each class.
\end{enumerate}

Both takeaways motivate our approach next.

\section{Our approach: Correct-\textsc{n}-Contrast (\name{})}
We now present \name{}, a two-stage method to improve worst-group performance and robustness to spurious correlations, \emph{without} training group labels. Similar to prior works, our first stage trains an ERM model with proper regularization\footnote{As we train on the same dataset we infer spurious attributes on, regularization (via high weight decay or early stopping) is to prevent the ERM model from memorizing train labels. This is standard (e.g. \citep{NEURIPS2020_e0688d13, liu2021just}). In Section~\ref{sec:results-ablation} we show that we do not require extremely accurate spurious attribute predictions 
to substantially improve robustness in practice.} on the training set, in order to infer spurious attributes.

The key difference is our second stage: we aim to train a more robust model by learning representations such that samples in the same class but different groups are close to each other. To do so, we use a contrastive learning strategy, proposing a new sampling scheme where we treat samples with the {same class} but {different spurious attributes} as distinct ``views'' {of the same class} (anchors and positives), while sampling anchors and negatives as samples with the same inferred spurious attributes but different classes. By training the second stage model with contrastive learning over these samples, we intuitively encourage this model to ``pull together'' samples' representations based on shared class-specific features, while ignoring different spurious features. To obtain such samples, we use the initial ERM model's predictions as spurious attribute proxies.

Later in Section~\ref{sec:results-main},~\ref{sec:results-representation}, we  show that \name{} indeed reduces $\hat{\mathcal{L}}_{\alig}(f_{\theta}; y)$ and substantially improves worst-group accuracy. In Section~\ref{sec:results-ablation} we also show that alternative sampling strategies degrade performance. We include further details on both stages below, and summarize \name{} in Algorithm~\ref{algo:cnc_algo}.

\label{sec:method}
\textbf{Stage 1: Inferring pseudo group labels.} 
We train an initial model $f_{\hat{\theta}}$ on the training dataset $\{(x_i, y_i)\}_{i=1}^n$ with ERM and regularization, and save its predictions $\{\hat{y}_i\}_{i=1}^n$ on the training datapoints.
We consider two ways to get predictions $\hat{y}$: (1) standard argmax over the ERM model's final-layer outputs (as in \citet{liu2021just}), and (2) clustering its last hidden-layer outputs into $C$ clusters\footnote{Recall that $C$ is the number of classes.} (similar to \citet{NEURIPS2020_e0688d13}). While both approaches exploit the ERM model's learned spurious correlations, we find clustering to lead to better performance  (cf.  Appendix~\ref{appendix:experimental_details_methods}).

\textbf{Stage 2: Supervised contrastive learning.}
Next, we train a robust model with supervised contrastive learning using the ERM predictions. Our approach is based on standard contrastive learning methods \citep{chen2020simple, NEURIPS2020_d89a66c7}, but we introduce new ``contrastive batch'' sampling and optimization objectives in order to induce robustness to spurious correlations.

\textbf{\textit{Contrastive batch sampling.}} As described in Section~\ref{sec:preliminaries}, contrastive learning involves sampling anchors, positives, and negatives with the general form $\{x\}, \{x^+\}, \{x^-\}$. Here, we wish to sample points such that by maximizing the similarity between anchors and positives (and keeping anchors and negatives apart), the Stage 2 model ``ignores'' spurious similarities while learning class-consistent dependencies. For each batch we randomly sample an \anc{} $x_i \in X$ (with label $y_i$ and ERM prediction $\hat{y}_i$), $M$ \pos{} with the same class as $y_i$ but a different ERM model prediction than $\hat{y}_i$, and $N$ \negs{} with different classes than $y_i$ but the same ERM model prediction as $\hat{y}_i$. 
For more comparisons per batch, we also switch anchor and positive roles (implementation details in Appendix~\ref{appendix:additional_algorithm_details_sampling}).

\begin{figure}[t!]
\begin{algorithm}[H]
\caption{Correct-\textsc{n}-Contrast (\name{})}
\label{algo:cnc_algo}
\begin{algorithmic}[1]
    \Input Training dataset $(X, Y)$; $\#$ positives $M$; $\#$ negatives $N$; learning rate $\eta$, \# epochs $K$. 
    \Statex{\textbf{Stage 1: Inferring pseudo group labels}}
    \State{Train ERM model $f_{\hat{\theta}}$ on ${(X, Y)}$; save $\hat{y}_i := f_{\hat{\theta}}(x_i)$.}
    \Statex{\textbf{Stage 2: Supervised contrastive learning}}
    \State{Initialize ``robust'' model $f_\theta$}
    \For{epoch $1,\dots,K$}
\For {\anc{} $(x, y) \in {(X, Y)}$}
    \State {Let $\hat{y} := f_{\hat{\theta}}(x)$ be the ERM model prediction of $x$.}
    \State {Get ${M}$ \pos{} $\set{(x_m^+, y_m^+)}$, where $y_m^+ = y$ and \underline{$\hat{y}_m^+ \neq \hat{y}$}, for $m = 1, \ldots, M$.}
    \State {Get ${N}$ \negs{} $\set{(x_n^-, y_n^-)}$, where $y_n^- \neq y$ and \underline{$\hat{y}_n^- = \hat{y}$},  for $n = 1, \ldots, N$.}
    \State {Update $f_\theta$ by $\theta \leftarrow \theta - \eta \cdot \nabla \hat{\cL}(f_{\theta}; x, y)$ (cf. Eq. \eqref{eq:full_loss}) with \anc{}, $M$ \pos{}, and $N$ \negs{}. \strut}
    \EndFor
    \EndFor
    \newline \Return final model $f_\theta$ from Stage 2.
\end{algorithmic}
\end{algorithm}
\vspace{-0.75cm}
\end{figure}

\textbf{\textit{Optimization objective and updating procedure.}}
Recall that we seek to learn aligned representations to improve robustness to spurious correlations. Thus, we also jointly train the full model to classify datapoints correctly. As we have the training \emph{class} labels, we jointly update the model's encoder layers $f_\text{enc}$ with a contrastive loss and the full model $f_\theta$ with a cross-entropy loss. Our overall objective is:
{
\setlength{\abovedisplayskip}{-2pt}
\setlength{\belowdisplayskip}{0pt}
{\small\begin{equation}
    \hat{\mathcal{L}}(f_\theta; x, y) = \lambda \hat{\mathcal{L}}_\text{con}^\text{sup}(f_\text{enc}; x, y) + (1 - \lambda) \hat{\mathcal{L}}_\text{cross}(f_\theta; x, y).
    \label{eq:full_loss} 
\end{equation}}}

Here $\hat{\mathcal{L}}_\text{con}^\text{sup}(f_\text{enc}; x, y)$ is the supervised contrastive loss of $x$ and its positive and negative samples, similar to Eq. \eqref{eq:supcon_easy_anchor} (see Eq \eqref{eq:full_contrastive_loss}, Appendix \ref{appendix:additional_algorithm_details_sampling} for the full equation);
$\hat{\mathcal{L}}_\text{cross}(f_\theta; x, y)$ is average cross-entropy loss over $x$, the $M$ positives, and $N$ negatives;
$\lambda {\in} [0, 1]$ is a balancing hyperparameter.

To calculate the loss, we first forward propagate one batch $\big(x_i, \{x_m^+\}_{m=1}^M, \{x_n^-\}_{n=1}^N\big)$ through $f_\text{enc}$ and normalize the outputs to obtain representation vectors $\big(z_i, \{z_m^+\}_{m=1}^M, \{z_n^-\}_{n=1}^N\big)$.  To learn closely aligned $z_i$ and $z_m^+$ for all $\{z_m^+\}_{m=1}^M$, we update $f_\text{enc}$ with the $\hat{\mathcal{L}}_\text{out}^\text{sup}(\cdot; f_\text{enc})$ loss. Finally, we also pass the unnormalized encoder outputs $f_\text{enc}$ to the classifier layers $f_{\text{cls}}$, and compute a batch-wise cross-entropy loss $\hat{\mathcal{L}}_\text{cross}(f_\theta)$ using each sample's class labels and $f_\theta$'s outputs.
Further details are in Appendix~\ref{appendix:additional_algorithm_details}.

\section{Experimental results}
\label{sec:exp_toplevel}
We conduct experiments to answer the following questions: (1) Does \name{} improve worst-group performance over prior state-of-the-art methods on datasets with spurious correlations? (2) To help explain any improvements, do \name{}-learned representations actually exhibit greater alignment and class-only dependence, and how is this impacted by the strength of a spurious correlation? 
(3) To better understand \name{}'s components and properties, how do ablations
on the Stage 1 prediction quality and the Stage 2 contrastive sampling strategy impact \name{} in practice?
We answer each question in the following sections. In Section~\ref{sec:results-main}, we show that \name{} substantially improves worst-group accuracy without group labels, averaging 3.6\% points higher than prior state-of-the-art. In Section~\ref{sec:results-representation}, we verify that more desirable representation metrics consistently coincide with these improvements. Finally in Section~\ref{sec:results-ablation}, we find that \name{} is more robust to inaccurate Stage 1 predictions than alternatives, improves worst-group accuracy by $0.9$ points over GDRO with group labels, and validates the importance of its sampling criteria.      
We present additional ablations on \name{}'s components, including the alignment approach, in Appendix \ref{appendix:add_benchmarks}.
We briefly describe the benchmark tasks we evaluate on below. Further details on datasets, models, and hyperparameters are in Appendix \ref{appendix:experimental_details}.

\textbf{Colored MNIST (\cmnist{})} \citep{arjovsky2019invariant}: We classify from $\mathcal{Y} = \{(0, 1), (2, 3), (4, 5), (6, 7), (8,9)\}$. We use $p_\text{corr} = 0.995$, so 99.5\% of samples are spuriously correlated with one of five colors $\in \mathcal{A}$.

\textbf{Waterbirds} \citep{
sagawa2019distributionally}: We classify from $\mathcal{Y} = \{\text{waterbird}, \text{landbird}\}$. 95\% of images have the same bird type $\mathcal{Y}$ and background type  $\mathcal{A} = \{\text{water}, \text{land}\}$.

\textbf{CelebA} \citep{liu2015deep}: We classify celebrities' hair color $\mathcal{Y} = \{\text{blond}, \text{not blond}\}$ with $\mathcal{A} = \{\text{male}, \text{female}\}$. Only 6\% of blond celebrities in the dataset are male. 

\textbf{CivilComments-WILDS} \citep{borkan2019nuanced, koh2021wilds}: We classify $\mathcal{Y} = \{\text{toxic}, \text{nontoxic}\}$ comments. $\mathcal{A}$ denotes a mention of one of eight demographic identities.

\begin{table*}[]
\caption{{Worst-group and average accuracies (over three seeds). \textbf{1st} / \underline{2nd} best worst-group accuracies \textbf{bolded} / \underline{underlined}. On image datasets, \name{} obtains substantially higher worst-group accuracy than comparable methods without group labels, competing with GDRO. \name{} also competes with SoTA on CivilComments. Starred results from original papers. Further implementation details in Appendix~\ref{appendix:experimental_details}.}}
\vspace{0.25cm}
\label{tab:main_results}
\centering
{
\tabcolsep=0.05cm
\begin{tabular}{@{}lbcbcbcbc@{}}
\toprule
\rowcolor{white}                     & \multicolumn{2}{c}{\cmnist{}}      & \multicolumn{2}{c}{Waterbirds}   & \multicolumn{2}{c}{CelebA}       & \multicolumn{2}{c}{CivilComments-WILDS} \\ \cmidrule(lr){2-3} \cmidrule(lr){4-5} \cmidrule(lr){6-7} \cmidrule(lr){8-9}

Accuracy (\%)                        & Worst-group         & Average       & Worst-group         & Average       & Worst-group         & Average       & Worst-group           & Average            \\ \midrule
ERM                             & 0.0 (0.0)             & 20.1 (0.2) & 62.6 (0.3)          & 97.3 (1.0) & 47.7 (2.1)          & 94.9 (0.3) & 58.6 (1.7)             & 92.1 (0.4)     \\
LfF                             & 0.0 (0.0)           & 25.0 (0.5) & 78.0 (-)$^*$           & 91.2$^*$   & 77.2 (-)$^*$            & 85.1 (-)$^*$   & 58.8 (-)$^*$               & 92.5 (-)$^*$       \\
\textsc{George}                          & \underline{76.4 (2.3)}          & 89.5 (0.3) & 76.2 (2.0)$^*$          & 95.7 (0.5)$^*$ & \,54.9 (1.9)$^*$          & 94.6 (0.2)$^*$ & -                      & -              \\
PGI                             & 73.5 (1.8)          & 88.5 (1.4) & 79.5 (1.9)          & 95.5 (0.8) & \underline{85.3 (0.3)} &	87.3 (0.1) & -                      & -              \\
CIM                             & 0.0 (0.0)           & 36.8 (1.3) & 77.2 (-)$^*$           & 95.6 (-)$^*$  & 83.6 (-)$^*$           & 90.6 (-) $^*$  & N/A                    & N/A            \\
EIIL                            & 72.8 (6.8)          & 90.7 (0.9) & 77.2 (1.0)          & 96.5 (0.2) & 81.7 (0.8)          & 85.7 (0.1) & 67.0 (2.4)$^*$             & 90.5 (0.2)$^*$     \\
\jtt{}         & 74.5 (2.4)          & 90.2 (0.8) & \underline{83.8 (1.2)}          & 89.3 (0.7) & 81.5 (1.7)          & 88.1 (0.3) & \textbf{69.3} (-)$^*$      & 91.1 (-)$^*$       \\
\name{} (Ours) & \textbf{77.4 (3.0)} & 90.9 (0.6) & \textbf{88.5 (0.3)} & 90.9 (0.1) & \textbf{88.8 (0.9)} & 89.9 (0.5) & \underline{68.9 (2.1)}             & 81.7 (0.5)     \\ \midrule
Group DRO                       & 78.5 (4.5)          & 90.6 (0.1) & 89.9 (0.6)          & 92.0 (0.6) & 88.9 (1.3)          & 93.9 (0.1) & 69.8 (2.4)             & 89.0 (0.3)    \\ \bottomrule
\end{tabular}
}
\vspace{-0.25cm}
\end{table*}

\subsection{Comparison of worst-group performance}
\label{sec:results-main}
To study (1), we evaluate \name{} on image classification and NLP tasks with spurious correlations. As baselines, we compare against standard ERM and an `oracle' GDRO approach that assumes access to the group labels. We also compare against recent methods that tackle spurious correlations without requiring group labels: \textsc{George} \citep{NEURIPS2020_e0688d13}, Learning from Failure (LfF) \citep{NEURIPS2020_eddc3427}, 
{Predictive Group Invariance (PGI) \citep{Ahmed2021SystematicGW}},
Environment Inference for Invariant Learning (EIIL) \citep{creager2021environment}, Contrastive Input Morphing (CIM) \citep{taghanaki2021robust},
and Just Train Twice (\textsc{Jtt}) \citep{liu2021just}. 
\name{} achieves \textbf{highest} worst-group accuracy among all methods without training group labels on \cmnist, Waterbirds, and CelebA, and near-SoTA worst-group accuracy on CivilComments-WILDS  (Table~\ref{tab:main_results}).

While LfF, \textsc{George}, 
PGI, EIIL, 
and \textsc{Jtt} similarly use a trained ERM model to estimate groups, \name{} uniquely uses ERM predictions to learn desirable representations via contrastive learning. By contrasting positives and negatives, we reason that \name{} more strongly encourages ignoring spurious attributes compared to prior invariance, input transformation, or upweighting approaches. We include additional support for this via GradCAM visualizations in Appendix~\ref{appendix:gradcams}.

\subsection{Detailed analysis of representation metrics}
\label{sec:results-representation}
To shed light on \name{}'s worst-group accuracy gains, we study if models trained with \name{} actually learn representations with higher alignment. Compared to ERM and \textsc{Jtt} (which obtained second highest worst-group accuracy on average), \name{} learns representations with significantly higher alignment (lower alignment loss)  and lower mutual information with spurious attributes, while having comparable mutual information with class labels (Fig.~\ref{fig:representation_effect_wbirds_celeba}).
This corresponds to \name{} models achieving the highest worst-group accuracy on Waterbirds and CelebA. Further, while all methods produce representations with high mutual information with class labels  (Fig.~\ref{fig:representation_effect_wbirds_celeba}b), compared to other methods, \name{} representations drastically reduce mutual information with spurious attributes (Fig.~\ref{fig:representation_effect_wbirds_celeba}c). In Fig.~\ref{fig:representation_visuals}, we further illustrate this  via UMAP visuals of the learned Waterbirds representations. All methods lead to class-separable representations. However, ERM's representations exhibit stronger separability by spurious attributes, and \textsc{Jtt}'s also have some spurious attribute dependency. \name{} uniquely learns representations that strongly depict class-only dependence.

\begin{figure*}[h]
  \centering
  \includegraphics[width=1\textwidth]{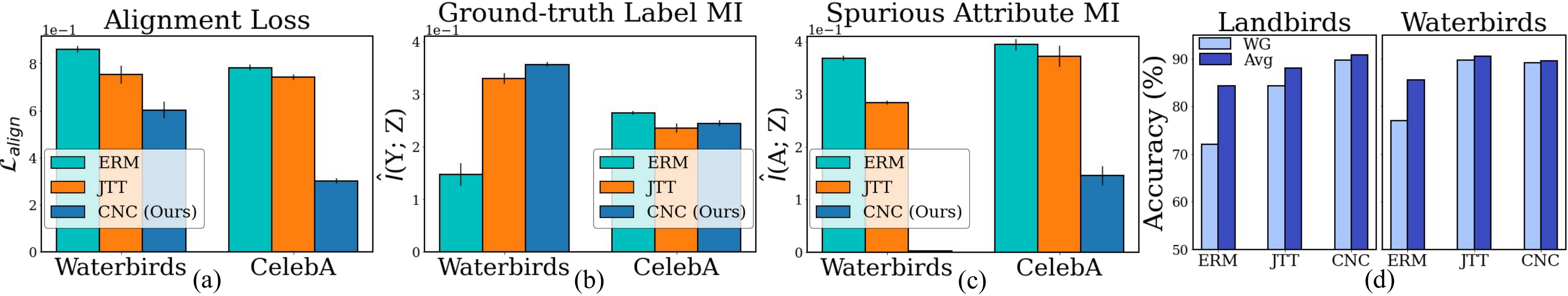}
  \vspace{-0.5cm}
  \caption{{Model alignment loss (a) and mutual information (b, c) after training with ERM, \textsc{Jtt}, and \name{}. \name{} most effectively reduces spurious attribute dependence, {and obtains smaller gaps for per-class worst-group versus average error (d), as supported by Theorem \ref{prop_close}}.}}
  \label{fig:representation_effect_wbirds_celeba}
  \vspace{-0.5cm}
\end{figure*}

\begin{figure*}[t!]
  \centering
  \includegraphics[width=1.0\textwidth]{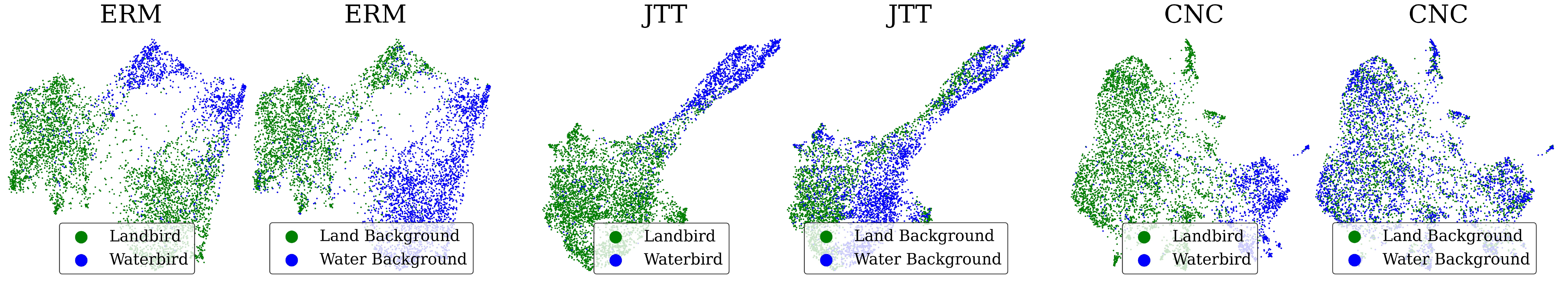}
  \vspace{-0.25cm}
  \caption{{Waterbirds model representations, visualized by UMAP. ERM representations heavily depend on both ground-truth $\mathcal{Y}$ and spurious $\cA$, with more separability by $\cA$. \textsc{Jtt} leads to greater separability by $\mathcal{Y}$, but still carries dependence on $\cA$. \name{} best removes  dependence on $\cA$.}}
  \label{fig:representation_visuals}
\end{figure*}

\begin{wrapfigure}{r}{0.53\textwidth}
  \centering
  \vspace{-0.25cm}
  \includegraphics[width=1\linewidth]{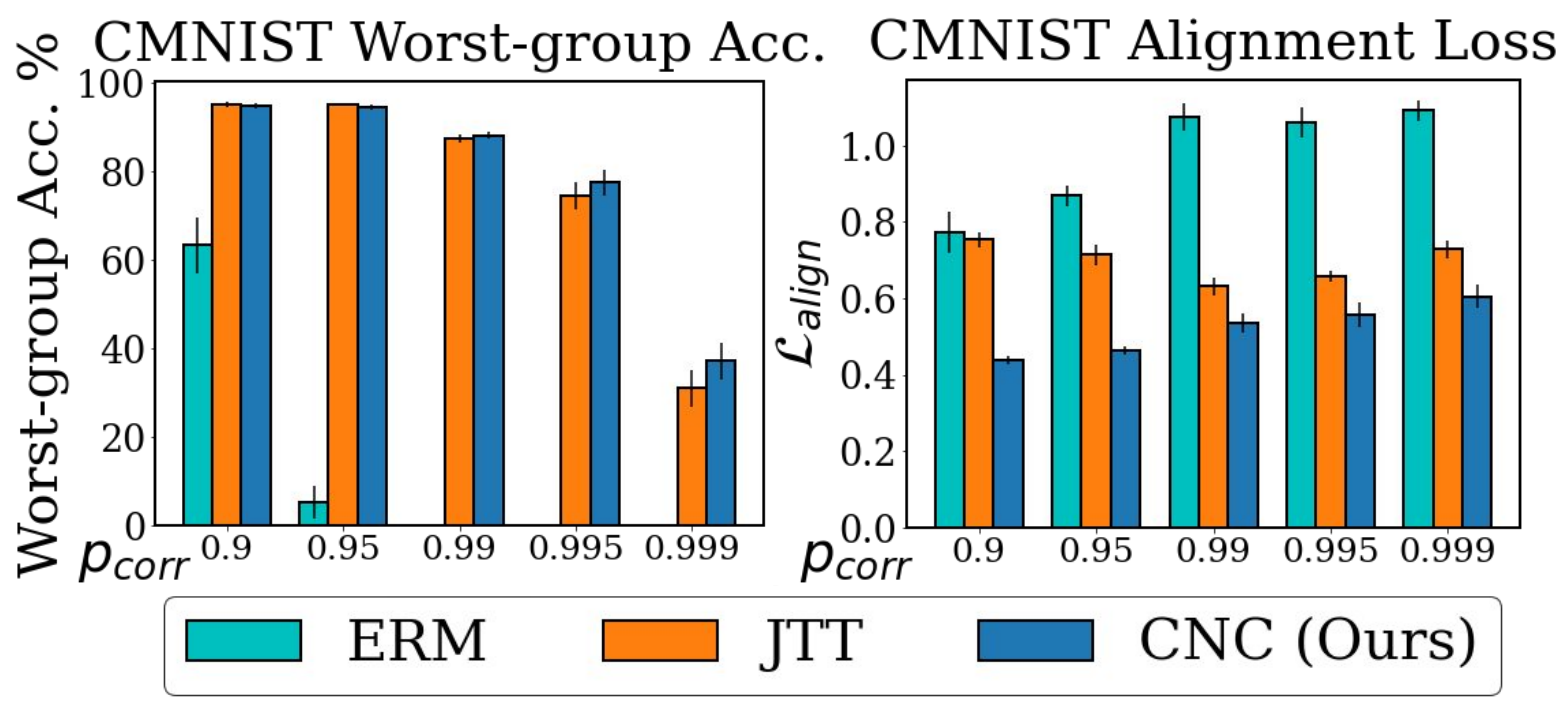}
  \vspace{-0.55cm}
  \caption{{Worst-group accuracy and alignment loss on \cmnist{} with increasing spurious correlation strength. \name{}'s higher accuracy coincides w/lower $\mathcal{L}_\text{align}$.}}
  \label{fig:representation_effect}
\end{wrapfigure}

In addition, to study how this relation between representation metrics and worst-group accuracy scales with the strength of the spurious correlation, we compute representation metrics with \textsc{CnC}, ERM, and \textsc{Jtt} models trained on increasingly spurious ($\uparrow p_{\text{corr}}$) \cmnist{} datasets (Fig.~\ref{fig:representation_effect}). While \name{} and \textsc{Jtt} maintain high worst-group accuracy where ERM fails, \name{} also performs better in more spurious settings ($p_\text{corr} >$  0.99). Lower alignment loss reflects these improvements (averaged over classes); \name{} consistently achieves lowest alignment loss. We report similar results for mutual information in Appendix~\ref{appendix:mutual_info_metrics_increasing_correlations}.

\subsection{Ablation studies}
\label{sec:results-ablation}
We study the importance of each individual component of our algorithm. We first study how \name{} is affected by how well the Stage 1 model predicts spurious attributes, finding \name{} more robust to noisy ERM predictions than comparable methods. We next evaluate \name{} with true spurious labels, finding \name{} to outperform GDRO with this group information. We finally ablate \name{}'s sampling criteria, and find our hard sampling strategy integral for substantial worst-group accuracy gains. In Appendix~\ref{appendix:loss_align_direct_comparison}, we compare other approaches for representation alignment, and find that ours achieve the highest worst-group accuracy.

\begin{wraptable}{r}{0.6\textwidth}
\vspace{-0.75cm}
\centering
{
\small
\caption{\cmnist{} Stage 2 model accuracies with noisy Stage 1 ``predictions''. \name{} is more robust than \textsc{Jtt} to increasing noise $p$. }
\vspace{0.25cm}
\tabcolsep=0.075cm
\label{tab:cmnist_noise_ablation}
\begin{tabular}{@{}lccccc@{}}
\toprule
\rowcolor{white}
\multicolumn{2}{l}{Noise Probability ($p$)} & 0                   & 0.01                & 0.1           & 0.25          \\ \midrule
\multirow{2}{*}{Worst-group Acc.}                 & JTT               & 78.4 \scriptsize{(2.3)}          & 70.2  \scriptsize{(4.7)}         & 53.8 \scriptsize{(6.0)}    & 43.6 \scriptsize{(7.3)}    \\
& CNC               & \textbf{80.6 \scriptsize{(2.8)}} & \textbf{76.5 \scriptsize{(4.1)}} & \textbf{61.3 \scriptsize{(6.0)}} & \textbf{53.2 \scriptsize{(7.0)}} \\
\midrule
\multirow{2}{*}{Average Acc.}                & JTT               & 89.7 \scriptsize{(0.9)}          & 86.7 \scriptsize{(1.3)}          & 79.5 \scriptsize{(1.6)}    & 72.9 \scriptsize{(1.4)}    \\
& CNC               & \textbf{92.1 \scriptsize{(0.5)}} & \textbf{89.4 \scriptsize{(1.4)}} & \textbf{81.7 \scriptsize{(2.0)}} & \textbf{74.9 \scriptsize{(2.0)}} \\ \bottomrule
\end{tabular}
}
\vspace{-0.205cm}
\end{wraptable}

\textbf{Influence of Stage 1 predictions.} ERM predictions can still be noisy proxies for spurious attributes; 
we thus study if a practical benefit of \name{}'s contrastive approach is more robustness to this noise compared to alternatives that also use ERM predictions, e.g., \jtt{}'s upsampling.
We run \name{} and \jtt{} with added noise to the same Stage 1 predictions---where \jtt{} upsamples the class-incorrect ones---and compare robust model accuracies. On \cmnist, starting with true spurious attribute labels as oracle Stage 1 ``predictions'', we add noise by replacing each label randomly with probability $p$. In Table~\ref{tab:known_group_label_comparison}, while both methods' accuracies drop as $p$ increases (the ``predictions'' degrade),  \name{} consistently achieves higher accuracy that degrades less than \jtt{}. 
On real data, \name{} also does not require perfectly inferring spurious attributes to work well. For Table~\ref{tab:main_results} Waterbirds and CelebA results, the Stage 1 ERM model predicts the spurious attribute with $94.7\%$ and $84.0\%$ accuracy respectively.

\newpage

\begin{wraptable}{r}{0.6\textwidth}
\vspace{-0.4cm}
\centering
{
\small
\caption{\small Accuracy (\%) using spurious attribute training labels. On average, \name{} obtains $0.9$ points higher worst-group (WG) over GDRO.}
\vspace{0.25cm}
\tabcolsep=0.025cm
\label{tab:known_group_label_comparison}
\begin{tabular}{@{}lbcbcbc@{}}
\toprule
\rowcolor{white}
 & \multicolumn{2}{c}{\cmnist{}}                                       & \multicolumn{2}{c}{Waterbirds}                                    & \multicolumn{2}{c}{CelebA}                                        \\
\cmidrule(lr){2-3} \cmidrule(lr){4-5} \cmidrule(lr){6-7}
 Acc.    & WG                                          & Avg.                & WG                                          & Avg.                & WG                                          & Avg.                \\

\midrule
\name{}*   & \textbf{80.6 \scriptsize{(2.8)}} & \textbf{92.4 \scriptsize{(0.2)}} & \textbf{90.1 \scriptsize{(0.2)} }         & \textbf{92.4 \scriptsize{(0.2)}} & \textbf{89.2 \scriptsize{(1.0)}} & 92.6 \scriptsize{(0.4)}          \\
GDRO   & 78.5 \scriptsize{(4.5)}          & 90.6 \scriptsize{(0.1)}          & 89.9 \scriptsize{(0.6)} & 92.0 \scriptsize{(0.2)} & 88.9 \scriptsize{(1.3)}          & \textbf{93.9 \scriptsize{(0.1)}} \\
\bottomrule
\end{tabular}
}
\vspace{-0.15cm}
\end{wraptable}

\textbf{Training with spurious labels.}
We next study \name{}'s performance with true spurious attribute labels on additional datasets. We replace the Stage 1 predictions with true group labels (denoted \name{}*) and compare with GDRO---the prior oracle baseline which uses group labels---in Table~\ref{tab:known_group_label_comparison}. We find \name{} improves with spurious labels, now obtaining $0.9\%$ and $0.2\%$ absolute lift in worst-group and average accuracy over GDRO, suggesting that \name{}'s contrastive approach can also be beneficial in settings with group labels.

\textbf{Alternate sampling strategies.} We finally study the importance of \name{}'s sampling by ablating individual criteria. Instead of sampling negatives by different class and same ERM prediction as anchors (\name{}), we try sampling negatives only with different classes (Different Class) or the same ERM prediction (Same Prediction) as anchors. We also try sampling both positives and negatives only by class, similar to \citet{NEURIPS2020_d89a66c7} (SupCon). Without both hard positive and negative sampling, we hypothesize naïve contrastive approaches could still learn spurious correlations (e.g., pulling apart samples different in spurious attribute and class by relying on spurious differences). In Table~\ref{tab:cnc_ablation_results}, we find these sampling ablations indeed result in lower worst-group accuracy, accompanied by less desirable representation metrics.

\begin{table*}[h]
\caption{Accuracy (\%) and representation metrics ($\times 0.1$) for \name{} sampling method ablations. Removing components of \name{}'s sampling criteria reduces worst-group (WG) acc., generally in line with higher $\mathcal{L}_\text{align}$, lower class dependence $\hat{I}(Y; Z)$, and higher spurious attribute dependence $\hat{I}(A; Z)$ than default \name{}.}
\vspace{0.25cm}
\label{tab:cnc_ablation_results}
\centering
{
\tabcolsep=0.5cm
\begin{tabular}{@{}lcbccc@{}}
\toprule
\rowcolor{white}
           & \multicolumn{5}{c}{Waterbirds}
           \\ \cmidrule(lr){2-6}
Method          &  Average Acc.    & WG Acc.     & $\mathcal{L}_\text{align}$ [$\downarrow$]    & $\hat{I}(Y; Z)$ [$\uparrow$]    & $\hat{I}(A; Z)$ [$\downarrow$]\\ \midrule
Different Class       & 
95.6 \normalsize{(0.1)} & 
82.2 \normalsize{(1.0)} & 

6.22 \normalsize{(0.12)} & 
\textbf{3.59 \normalsize{(0.05)}} & 
1.60 \normalsize{(0.13)} \\ 

Same Prediction        & 
93.6 \normalsize{(0.5)} & 
86.1 \normalsize{(0.5)} &  

6.67 \normalsize{(0.13)} & 
3.50 \normalsize{(0.13)} & 
0.16 \normalsize{(0.09)} \\

SupCon           & 
\textbf{96.8 \normalsize{(0.2)}} & 
62.3 \normalsize{(2.2)} & 

6.93 \normalsize{(0.44)} & 
\textbf{3.59 \normalsize{(0.06)}} & 
4.84 \normalsize{(0.13)} \\

\name{}                 & 

90.9 \normalsize{(0.1)} & 
\textbf{88.5 \normalsize{(0.3)}} & 

\textbf{6.02 \normalsize{(0.35)}} & 
3.56 \normalsize{(0.04)} & 
\textbf{0.02 \normalsize{(0.01)}} \\
\bottomrule
\end{tabular}
}
\end{table*}

\begin{table*}[h]
\vspace{-0.4cm}
\label{tab:cnc_ablation_results}
\centering
{
\tabcolsep=0.5cm
\begin{tabular}{@{}lcbccc@{}}
\toprule
\rowcolor{white}
           & \multicolumn{5}{c}{CelebA}                                       \\ \cmidrule(lr){2-6}
Method          &  Average Acc.    & WG Acc.     & $\mathcal{L}_\text{align}$ [$\downarrow$]    & $\hat{I}(Y; Z)$ [$\uparrow$]    & $\hat{I}(A; Z)$ [$\downarrow$]  \\ \midrule
Different Class       & 
89.5 \normalsize{(0.1)} &
79.2 \normalsize{(0.4)} & 
 
3.45 \normalsize{(0.04)} & 
2.34 \normalsize{(0.03)} & 
2.38 \normalsize{(0.06)} \\

Same Prediction        & 
88.2 \normalsize{(1.3)} &
75.0 \normalsize{(5.9)} & 
 
3.37 \normalsize{(0.59)} & 
2.08 \normalsize{(0.04)} & 
2.13 \normalsize{(0.53)} \\

SupCon           & 
\textbf{90.4 \normalsize{(0.5)}} & 
61.5 \normalsize{(2.0)} & 

3.83 \normalsize{(0.23)} & 
2.29 \normalsize{(0.04)} & 
3.22 \normalsize{(0.13)} \\

\name{}                 & 

89.9 \normalsize{(0.5)} &
\textbf{88.8 \normalsize{(0.9)}} & 
 
\textbf{3.00 \normalsize{(0.12)}} & 
\textbf{2.44 \normalsize{(0.06)}} & 
\textbf{1.45 \normalsize{(0.18)}} \\
\bottomrule
\end{tabular}
}
\end{table*}

\section{Related work}\label{sec_related}

There is a growing literature on how to improve robustness to spurious correlations, which is a key concern in many settings due to their resulting inherent bias against minority groups. If group labels are known, prior works often design a method to {balanced groups of different sizes}, whether via class balancing \citep{he2009learning, cui2019class}, importance weighting \citep{shimodaira2000improving, byrd2019effect}, or robust optimization \citep{sagawa2019distributionally}.

Our work is more related to methods that do not require  group labels during training. Such methods commonly first train an initial ERM model, before using this model to train a second robust model. \textsc{George} \citep{NEURIPS2020_e0688d13} run GDRO with groups formed by clustering ERM representations. LfF \citep{NEURIPS2020_eddc3427} and \textsc{Jtt} \cite{liu2021just} train a robust model by upweighting or upsampling the misclassified points of an ERM model. EIIL \citep{creager2021environment} and PGI \citep{Ahmed2021SystematicGW} infer groups that maximally violate the invariant risk minimization (IRM) objective \citep{arjovsky2017towards} for the ERM model. With these groups EIIL trains a robust model with GDRO, while PGI minimizes the KL divergence of softmaxed logits for same-class samples across groups. CIM \citep{taghanaki2021robust} instead trains a transformation network to remove potentially spurious attributes from image input features. While these approaches can also encourage alignment, our approach more directly acts on a model's representations via contrastive learning. \name{} thus leads to better alignment {empirically} as measured by our representation metrics.

Our proposed algorithm draws inspiration from the literature on self-supervised contrastive learning. which works by predicting whether two inputs are ``similar'' or ``dissimilar'' \citep{le2020contrastive}. This involves specifying batches of \textit{anchor} and \textit{positive} datapoints similar to each other (as different ``views'' of the same source or input), and \textit{negatives} depicting dissimilar points. 
In contrastive learning, ``negatives'' are often sampled uniformly \citep{NEURIPS2019_ddf35421}, while ``positives'' are different views of the same object, e.g. via data augmentation \citep{chen2020simple}. In supervised contrastive learning, negatives are different-class points and positives are same-class points \citep{NEURIPS2020_d89a66c7}. Our approach  treats same-class points with different ERM predictions as positives, and different-class points with the same ERM prediction as negatives. This naturally provides a form of \textit{hard negative mining}, a nontrivial component of recent contrastive learning shown to improve performance \citep{Robinson2020ContrastiveLW, Wu2020ConditionalNS, Chuang2020DebiasedCL}. Finally, \name{} is also partly inspired by \citet{wang2020understanding}, who show that minimizing the contrastive loss improves representation alignment between distinct ``views.''

We refer the reader to Appendix \ref{sec:extended_related_work} for further discussion of other related works.

\section{Conclusion}
We present \name{}, a two-stage contrastive learning approach to learn representations robust to spurious correlations.
We empirically observe, and theoretically analyze, the connection between alignment and worst-group versus average-group losses.
\name{} improves the quality of learned representations by making them more class-dependent and less spurious-attribute-dependent, and achieves state-of-the-art or near-state-of-the-art worst-group accuracy across several benchmarks.

\section*{Acknowledgments}
We thank Jared Dunnmon, Karan Goel, Megan Leszczynski, Khaled Saab, Laurel Orr, and Sabri Eyuboglu for helpful discussions and feedback.

We gratefully acknowledge the support of NIH under No. U54EB020405 (Mobilize), NSF under Nos. CCF1763315 (Beyond Sparsity), CCF1563078 (Volume to Velocity), and 1937301 (RTML); ARL under No. W911NF-21-2-0251 (Interactive Human-AI Teaming); ONR under No. N000141712266 (Unifying Weak Supervision); ONR N00014-20-1-2480: Understanding and Applying Non-Euclidean Geometry in Machine Learning; N000142012275 (NEPTUNE); Apple, NXP, Xilinx, LETI-CEA, Intel, IBM, Microsoft, NEC, Toshiba, TSMC, ARM, Hitachi, BASF, Accenture, Ericsson, Qualcomm, Analog Devices, Google Cloud, Salesforce, Total, the HAI-GCP Cloud Credits for Research program, the Stanford Data Science Initiative (SDSI),
and members of the Stanford DAWN project: Facebook, Google, and VMWare. The U.S. Government is authorized to reproduce and distribute reprints for Governmental purposes notwithstanding any copyright notation thereon. Any opinions, findings, and conclusions or recommendations expressed in this material are those of the authors and do not necessarily reflect the views, policies, or endorsements, either expressed or implied, of NIH, ONR, or the U.S. Government.

\newpage
\bibliography{main}
\bibliographystyle{plainnat}
\balance

\newpage
\appendix

\input{appendix}

\end{document}

%% file: analysis.tex
\subsection{Justification that better alignment encourages lower worst-group loss}\label{sec_analysis}

Next, we give a rigorous justification of the relation from lower alignment loss to lower worst-group loss.
For any label $y\in\cY$, let $\cG_y$ be the set of groups with label $y$ in $\cG$.
\mbox{Let $\cL_{\wg}(f_{\theta}; y)$ be the worst-group loss  in {\small $\cG_y$}:}
{
\begin{equation*}
\cL_{\wg}(f_{\theta}; y) \define \max_{g\in\cG_y} \exarg{(x, \tilde y, a) \sim P_g}{ \ell(f_{\theta}(x), \tilde y) }.
\end{equation*}
}
Let $\cL_{\avg}(f_{\theta}; y)$ be the average loss among groups in $\cG_y$:
{
\begin{equation*}
\cL_{\avg}(f_{\theta}; y) \define \exarg{(x, \tilde y, a)\sim P : \forall a\in\cA}{\ell(f_{\theta}(x), \tilde y)}. 
\end{equation*}}
Additionally,
let {$\hat \cL_{\alig}(f_{\enc}; y)$} be the largest cross-group alignment loss among groups in $\cG_y$:
{
\begin{align}
    \hat{\cL}_{\alig}(f_{\theta}; y) \define \max_{g\in\cG_y, g'\in\cG_y:\, g\neq g'} \hat{\cL}_{\alig}(f_{\enc}; g, g'). \label{eq_cross_group}
\end{align}}
We state our result as follows.
\begin{theorem}\label{prop_close}
    In the setting described above,
    suppose the weight matrix of the linear classification layer $W$ satisfies  $\norm{W}_2 \le B$, for some  $B > 0$.
    Suppose the loss function $\ell(x, y)$ is $C_1$-Lipschitz in $x$ and bounded from above by $C_2$, for some $C_1 >0$ and $C_2 > 0$.
    Let $n_{g}$ be the size of any group $g\in\cG$ in the training set.
    Then, for any $\delta > 0$, with probability $1-\delta$, the following holds for any $y \in \cY$:
    {\begin{align}
        &\mathcal{L}_{\text{wg}}(f_{\theta}; y) 
        - \mathcal{L}_{\text{avg}}(f_{\theta}; y) 
        \le
        B C_1 \cdot \hat{\cL}_{\alig}(f_{\theta}; y) + 
         \max_{g\in\mathcal{G}_y} C_2\sqrt{ 8\log(|\cG_y| /\delta) / n_g}. \label{eq_result_align} 
    \end{align}}
\end{theorem}

The broader implication of our result is that reducing the alignment loss closes the gap between worst-group  and average-group losses.
The proof is deferred to Appendix \ref{sec_proof_align}.

%% file: appendix.tex
\section{Contrastive algorithm design details}\label{appendix:additional_algorithm_details}

In this section, we provide further details on the training setup and contrastive batch sampling, algorithmic pseudocode, and additional components related to \name{}'s implementation.

\subsection{Training setup}\label{appendix:additional_algorithm_details_training_setup}
In Fig.~\ref{fig:method_overview}, we illustrate the two training stages of Correct-\textsc{n}-Contrast described in Sec.~\ref{sec:method}. In Stage 1, we first train an ERM model with a cross-entropy loss. For consistency with Stage 2, we depict the output as a composition of the encoder and linear classifier layers. Then in Stage 2, we train a new model with the same architecture using contrastive batches sampled with the Stage 1 ERM model and a supervised contrastive loss \eqref{eq:supcon_easy_anchor} (which we compute after the depicted representations are first normalized) to update the encoder layers. Note that unlike prior work in contrastive learning \citep{chen2020simple, NEURIPS2020_d89a66c7}, as we have the class labels of the anchors, positives, and negatives, we also continue forward-passing the unnormalized representations (encoder layer outputs) and compute a cross-entropy loss to update the classifier layers while jointly training the encoder. 

\begin{figure}[h]
  \vspace{-0.25cm}
  \centering
  \includegraphics[width=1\textwidth]{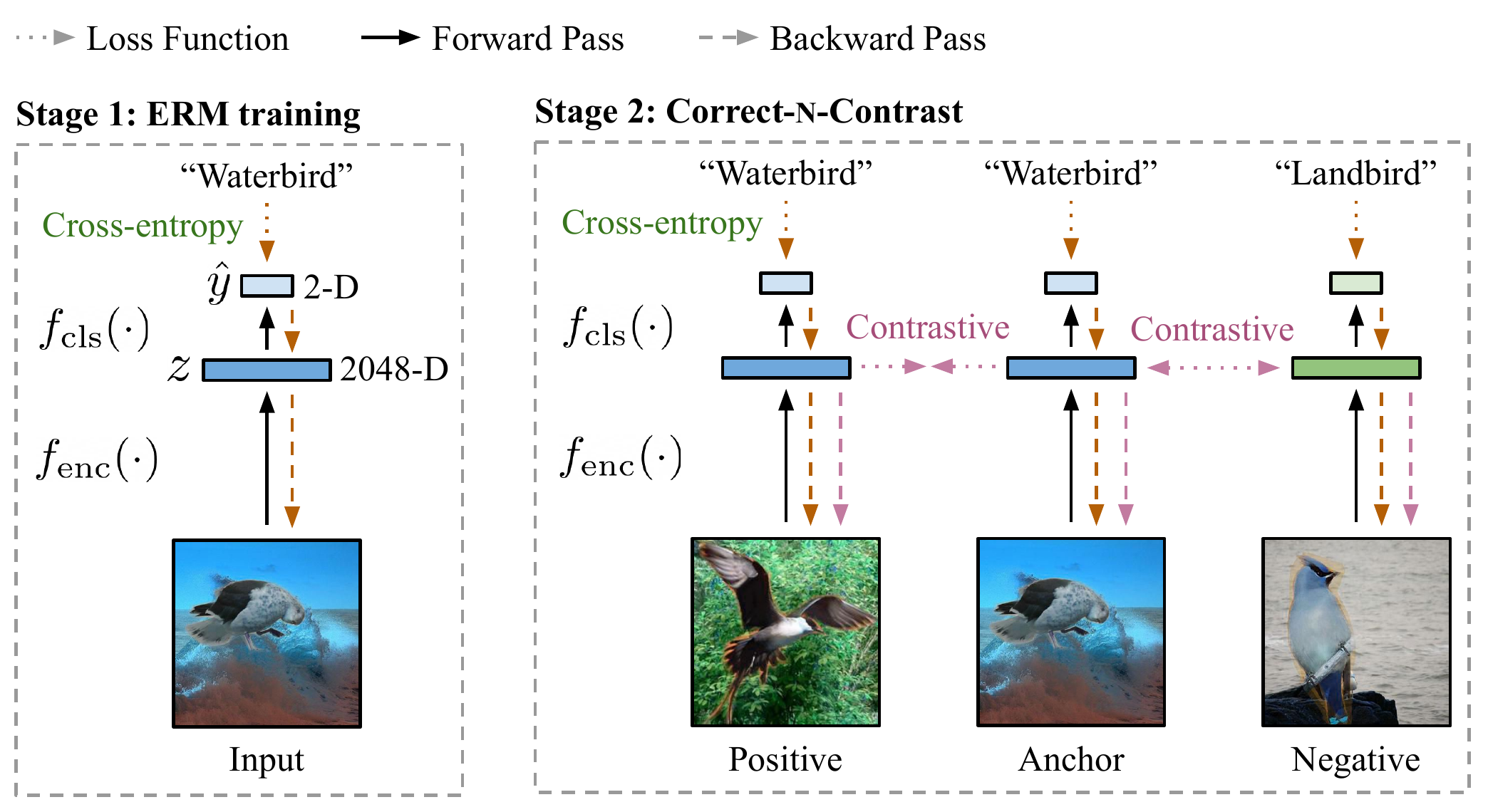}
     \vspace{-0.75cm}
  \caption{The two stages of Correct-\textsc{n}-Contrast. In Stage 1, we train a model with standard ERM and a cross-entropy loss. Then in Stage 2, we train a new model with the same architecture, but specifically learn spurious-attribute-invariant representations with a contrastive loss \eqref{eq:supcon_easy_anchor} and batches of anchors, positives, and negatives sampled with the ERM model's predictions. We also update the full model jointly with a cross-entropy loss on the classifier layer output and the input class labels. Dimensions for ResNet-50 and Waterbirds.}
  \label{fig:method_overview}
\end{figure}

We also note that unlike prior work, we wish to learn invariances between anchors and positives that maximally reduce the presence of features not needed for classification. We thus do not pass the representations through an additional \textit{projection network} \citep{chen2020simple}. Instead, we use Eq.~\ref{eq:supcon_easy_anchor} to compute the supervised contrastive loss directly on the encoder outputs $z = f_\text{enc}(x)$.

\subsection{Two-sided contrastive batch implementation} \label{appendix:additional_algorithm_details_sampling}

We provide more details on our default contrastive batch sampling approach described in Sec.~\ref{sec:method}. To recall, for additional contrastive signal per batch, we can double the pairwise comparisons in a training batch by switching the anchor and positive roles. This is similar to the \textit{NT-Xent} loss in prior contrastive learning work \citep{chen2020simple}. We switch the role of the anchor and first positive sampled in a contrastive batch, and sample additional positives and negatives using the same guidelines but adjusting for the ``new'' anchor. We denote this as ``two-sided'' sampling in contrast with the ``one-sided'' comparisons we get with just the original anchor, positives, and negatives. 

\begin{algorithm}[H]
\caption{Sampling two-sided contrastive batches}
\begin{spacing}{1.25}
\begin{algorithmic}[1]
    \Require Number of positives $M$ and number of negatives $N$ to sample for each batch.
    \State{Initialize set of contrastive batches $B = \{ \}$ \;}
    \For{$x_i \in \{x_i \in X : \hat{y}_i = y_i \}$ \;}
    \State{Sample $M - 1$ additional ``anchors'' to obtain $\{x_i\}_{i=1}^M$ from $\{x_i \in X : \hat{y}_i = y_i\}$ \;}
    \State{Sample $M$ positives $\{x_m^+\}_{m=1}^M$ from $\{x_m^- \in X : \hat{y}_m^- = \hat{y}_i,\; y_m^- \neq y_i\}$\;}
    \State{Sample $N$ negatives $\{x_n^-\}_{n=1}^N$ from $\{x_n^- \in X : \hat{y}_n^- = \hat{y}_i,\; y_n^- \neq y_i\}$\;}
    \State{Sample $N$ negatives $\{x_{n}'^{-}\}_{n=1}^N$ from $\{x_{n}'^{-} \in X : \hat{y}_{n}'^{-} = \hat{y}_1^+, y_{n}'^{-} \neq y_1^+ \}$\;}
    \State{Update contrastive batch set: $B \leftarrow B \cup \Big(\{x_i\}_{i=1}^M, \{x_m^+\}_{m=1}^M, \{x_n^-\}_{n=1}^N, \{x_{n}'^{-}\}_{n=1}^N\Big)$}
    \EndFor
\end{algorithmic}
\end{spacing}
\label{algo:two_sided_contrastive_batch}
\end{algorithm}

Implementing this sampling procedure in practice is simple. First, recall our initial setup with trained ERM model $f_{\hat{\theta}}$, its predictions $\{\hat{y}_i\}_{i=1}^n$ on training data $\{(x_i, y_i)\}_{i=1}^n$ (where $\hat{y}_i = f_{\hat{\theta}}(x_i)$), and number of positives and negatives to sample $M$ and $N$. We then sample batches with Algorithm~\ref{algo:two_sided_contrastive_batch}.

Because the initial anchors are then datapoints that the ERM model gets correct, under our heuristic we infer $\{x_i\}_{i=1}^M$ as samples from the majority group. Similarly the $M$ positives $\{x_m^+\}_{m=1}^M$ and $N$ negatives $\{x_n^-\}_{n=1}^N$ that it gets incorrect are inferred to belong to minority groups. 

For one batch, we then compute the full contrastive loss with
\begin{equation}
    \hat{\mathcal{L}}_\text{con}^\text{sup}(f_\text{enc}) =  \hat{\mathcal{L}}_\text{con}^\text{sup} \left(x_1, \{x_m^+\}_{m=1}^M, \{x_n^-\}_{n=1}^N; f_\text{enc}\right) + \hat{\mathcal{L}}_\text{con}^\text{sup} \left(x_1^+, \{x_i\}_{i=1}^M, \{x_n'^-\}_{n=1}^N; f_\text{enc}\right)
    \label{eq:full_contrastive_loss}
\end{equation}
where $\hat{\mathcal{L}}_\text{con}^\text{sup} \left(x_1, \{x_m^+\}_{m=1}^M, \{x_n^-\}_{n=1}^N; f_\text{enc}\right)$ is given by:
\begin{equation}
    - \frac{1}{M} \sum_{m=1}^M \log \frac{\exp( z_1^{\top} z_m^+ / \tau)}{\sum_{m=1}^M \exp( z_1^{\top} z_m^+ / \tau) + \sum_{n=1}^N \exp( z_1^{\top} z_n^+ / \tau)} 
    \label{eq:supcon_leftside}
\end{equation}

and again let $z$ be the normalized output $f_\text{enc}(x)$ for corresponding $x$. We compute the cross-entropy component of the full loss for each $x$ in the two-sided batch with its corresponding label $y$.

\begin{figure}[H]
  \vspace{-0.25cm}
  \centering
  \includegraphics[width=0.8\textwidth]{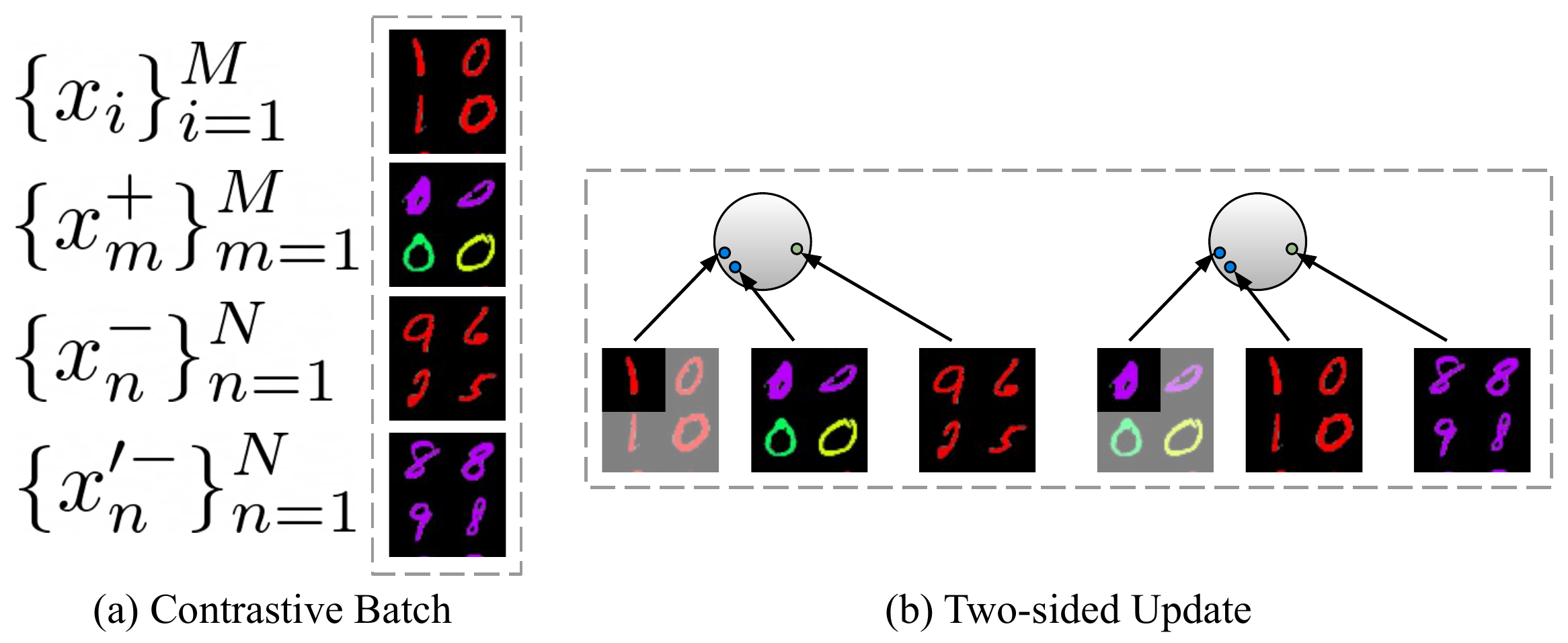}
  \caption{Illustration of two-sided contrastive batch sampling with Colored MNIST as an example. From a single batch (a), we can train a contrastive model with two anchor-positive-negative pairings (b). Aside from increasing the number of ``hard negatives'' for each anchor-positive pair, this intuitively ``pushes'' together anchors and positives from two different directions for greater class separation.}
  \label{fig:erm_prediction_comparison}
\end{figure}

\newpage

\subsection{Summary of \name{} design choices and properties} \label{appendix:additional_algorithm_details_summary}

\textbf{No projection network.} As we wish to learn data representations that maximize the alignment between anchor and positive datapoints, we do not compute the contrastive loss with the outputs of an additional nonlinear projection network. This is inspired by the logic justifying a projection head in prior contrastive learning, e.g. SimCLR \citep{chen2020simple}, where the head is included because the contrastive loss trains representations to be ``invariant to data transformation'' and may encourage removing information ``such as the color or orientation of objects''. In our case, we view inferred datapoints with the same class but different spurious attributes as ``transformations'' of each other, and we hypothesize that removing these differences can help us improve worst-group performance.

\textbf{Two-sided contrastive sampling.} To incorporate additional comparisons between datapoints that only differ in spurious attribute during training, we employ ``two-sided'' contrastive batch sampling. This lets us equally incorporate instances where the second contrastive model in \name{} treats datapoints that the initial ERM model got incorrect and correct as anchors. 

\textbf{Additional intrinsic hard positive/negative mining.} Because the new model corrects for potentially learned spurious correlations by only comparing and contrasting datapoints that differ in class label or spurious attribute, but not both (as dictated by the initial ERM model's outputs), the contrastive batches naturally carry ``hard'' positives and negatives. Thus, our approach provides a natural form of hard negative mining (in addition to the intrinsic hard positive and negative mining at the gradient level with InfoNCE-style losses \citep{chen2020simple, NEURIPS2020_d89a66c7}) while avoiding class collisions, two nontrivial components in self-supervised contrastive learning \citep{Robinson2020ContrastiveLW, Wu2020ConditionalNS, Chuang2020DebiasedCL}.

\textbf{Joint training of encoder and classifier layers.} \name{} can be used to train any standard classification model architecture; for any given neural network we just apply different optimization objectives to the encoder and classifier layers. We train both the encoder and classifier layers with a cross-entropy loss, and jointly train the encoder layer with a supervised contrastive loss. For the encoder layers, we balance the two objectives with a hyperparameter $\lambda$ (c.f. Eq.~\ref{eq:full_loss}).

\input{proof}

\section{{Additional empirical comparisons and ablations}}\label{appendix:add_benchmarks}
In this section, we include further experiments comparing \name{} against additional related methods and ablations to study the importance of \name{}'s presented design choices. We first consider an alternative representation alignment procedure by minimizing the presented alignment loss directly as opposed to using contrastive learning in Appendix~\ref{appendix:loss_align_direct_comparison}. We next report additional mutual information metrics from our experiments in Section~\ref{sec:results-representation}, where we study the relation between these representation metrics and worst-group accuracy on datasets with increasingly strong spurious correlations. We then summarize and empirically ablate various other design choices for \name{} in Appendix~\ref{appendix:additional_design_choice_ablations}. In Appendix~\ref{appendix:dg_comparison}, we finally compare \name{} against related work in representation learning for unsupervised domain adapation---a setting that similarly involves overcoming group-specific dependencies---once properly adapted for our spurious correlations setting.  

\subsection{{Comparison to minimizing the alignment loss directly}}\label{appendix:loss_align_direct_comparison}
In Sec.~\ref{sec:results-main} and Sec.~\ref{sec:results-representation}, we empirically showed that \name{}'s contrastive loss and hard positive and negative sampling lead to improved worst-group accuracy and greater representation alignment. While the analysis in \citet{wang2020understanding} also discusses how contrastive learning supports alignment between anchor and positives, we also investigate  how \name{} performs if instead of the contrastive loss, we train the Stage 2 model to minimize $\mathcal{L}_\text{align}$ directly as a training objective. With this objective, we aim to minimize the Euclidean distance between samples in different inferred groups but the same class. While \name{} is motivated by improving alignment, we hypothesize that one advantage of the contrastive loss lies in not only aligning anchors and positives, but also \textit{pulling apart} hard negatives, improving class-separability. We keep all other components consistent, and apply $\mathcal{L}_\text{align}$ to the anchor and positive samples in each contrastive batch. We report results on Waterbirds and CelebA in Table~\ref{tab:loss_align_comparison}. 

\begin{table}[h!]
\vspace{-0.25cm}
\caption{Accuracy (\%) and representation metrics ($\times 0.1$) comparing \name{} as proposed vs. $\mathcal{L}_\text{align}$ as the training objective. While the latter results in lower alignment, it does not encourage separating hard negatives from anchors. This results in representations with lower mutual information with class labels and higher mutual information with spurious attributes, and lower worst-group and average accuracies. }
\vspace{0.25cm}
\label{tab:loss_align_comparison}
\centering
{
\tabcolsep=0.5cm
\begin{tabular}{@{}lcbccc@{}}
\toprule
           & \multicolumn{5}{c}{Waterbirds}         \\ \cmidrule(lr){2-6} 
Method          &  Average Acc.    & WG Acc.     & $\mathcal{L}_\text{align}$ [$\downarrow$]    & $\hat{I}(Y; Z)$ [$\uparrow$]    & $\hat{I}(A; Z)$ [$\downarrow$]  \\ \midrule
\name{} ($\mathcal{L}_\text{align}$)  & 
82.3 \scriptsize{(0.1)} & 
88.9 \scriptsize{(0.0)} & 
\textbf{2.16 \scriptsize{(0.01)}} & 
2.52 \scriptsize{(0.06)} & 
3.27 \scriptsize{(0.02)} \\

\name{}                 & 

\textbf{90.9 \scriptsize{(0.1)}} & 
\textbf{88.5 \scriptsize{(0.3)}} & 

6.02 \scriptsize{(0.35)} & 
3.56 \scriptsize{(0.04)} & 
\textbf{0.02 \scriptsize{(0.01)}} \\
\bottomrule
\end{tabular}
}
\end{table}
\begin{table}[h!]
\vspace{-0.4cm}
\label{tab:loss_align_comparison}
\centering
{
\tabcolsep=0.5cm
\begin{tabular}{@{}lcbccc@{}}
\toprule
           & \multicolumn{5}{c}{CelebA}                                       \\ \cmidrule(lr){2-6}  
Method          &  Average Acc.    & WG Acc.     & $\mathcal{L}_\text{align}$ [$\downarrow$]    & $\hat{I}(Y; Z)$ [$\uparrow$]    & $\hat{I}(A; Z)$ [$\downarrow$]  \\ \midrule
\name{} ($\mathcal{L}_\text{align}$)  & 
82.3 \scriptsize{(1.6)} & 
85.9 \scriptsize{(0.8)} & 
\textbf{2.59 \scriptsize{(0.20)}} & 
2.00 \scriptsize{(0.03)} & 
1.59 \scriptsize{(0.30)} \\

\name{}                 & 

\textbf{89.9 \scriptsize{(0.5)}} &
\textbf{88.8 \scriptsize{(0.9)}} & 
 
3.00 \scriptsize{(0.12)} & 
\textbf{2.44 \scriptsize{(0.06)}} & 
\textbf{1.45 \scriptsize{(0.18)}} \\
\bottomrule
\end{tabular}
}
\end{table}

We find that \name{} with the default contrastive loss outperforms \name{} with the alignment loss in both worst-group and average accuracy, and that this indeed corresponds with representations that exhibit lower mutual information with class labels. We reason that the additional pushing apart of hard negative samples (with different class labels but similar spurious features) provides additional signal for improving separation between the different classes. The robust model thus only learns to rely on class-specific features for discriminating between datapoints. However, the $\mathcal{L}_\text{alignment}$ objective does not incorporate these hard negatives. 

\subsection{Extended analysis of representation metrics over increasing spurious correlations}\label{appendix:mutual_info_metrics_increasing_correlations}

In Section~\ref{sec:results-representation}, we found that \name{}'s improved worst-group accuracy on popular spurious correlations benchmarks coincided with representations with lower alignment loss, higher mutual information with ground-truth classes, and lower mutual information with spurious attributes. To study how this relation between representation metrics and worst-group accuracy scaled with the strength of the spurious correlation, we also computed representation metrics over increasingly spurious \cmnist{} datasets for \textsc{CnC}, ERM, and \textsc{Jtt}. Below in Fig.~\ref{fig:representation_effect_extended} we reproduce Fig.~\ref{fig:representation_effect} with added mutual information metrics.

\begin{figure}[H]
  \vspace{-0.25cm}
  \centering
  \includegraphics[width=1\textwidth]{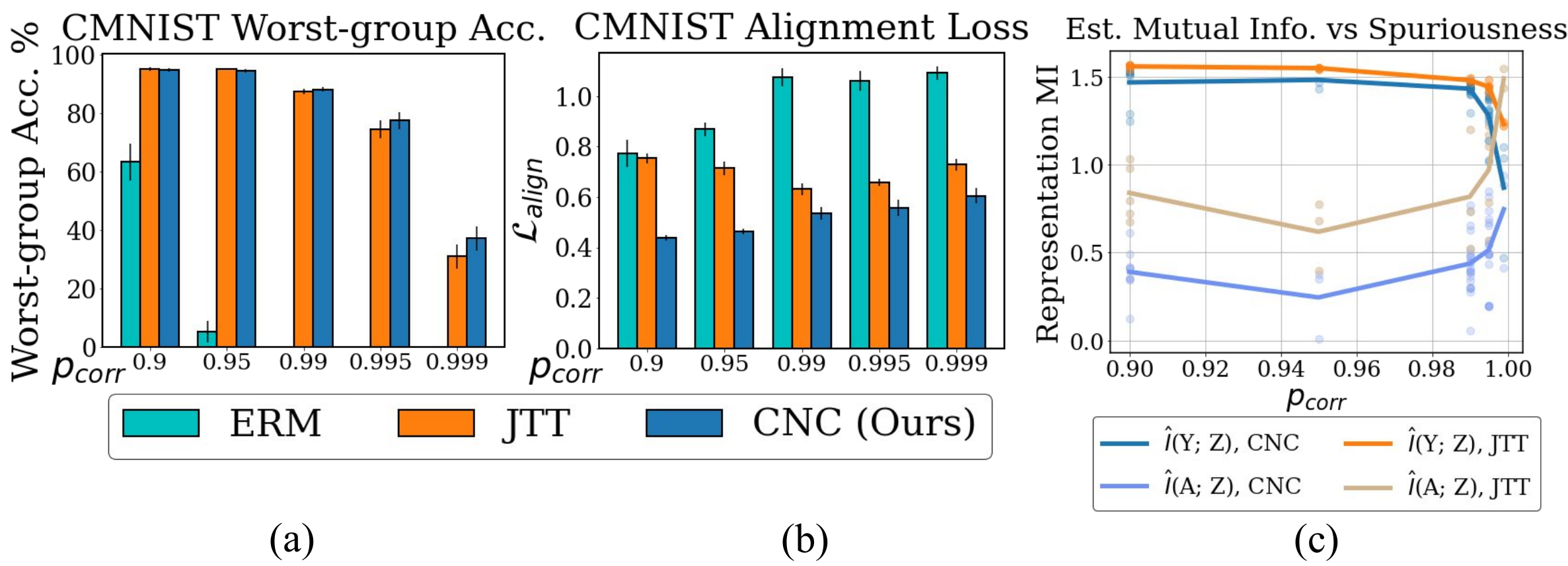}
  \vspace{-0.75cm}
  \caption{Alignment loss and mutual information metrics with worst-group accuracy on increasingly spurious \cmnist. \name{}'s highest worst-group accuracy in spuriously correlated datasets (a) coincides with learning representations with better alignment (b), and a more desirable ratio of mutual information dependence on the ground-truth class labels vs the spurious attribute (c).}
  \label{fig:representation_effect_extended}
  \vspace{-0.5cm}
\end{figure}

Fig.~\ref{fig:representation_effect_extended}(c) shows that \name{}'s learned representations maintain a more favorable balance of mutual information between the class label and spurious attribute than \textsc{Jtt}. While \textsc{Jtt} models exhibit slightly higher estimated $\hat{I}(Y; Z)$ than \name{} models, \name{} models exhibit much lower dependence on the spurious attribute. In the regime where $99.9\%$ of digits are spuriously correlated with a specific color, \name{} uniquely learns representations with higher mutual information on class labels than spurious attributes.

\subsection{Additional design choice ablations}\label{appendix:additional_design_choice_ablations}

To validate the additional algorithmic components of \name{}, we report how \name{} performs on the Waterbirds dataset when modifying the individual design components. We use the same hyperparameters as in the main results, and report accuracies as the average over three training runs for the following ablations. Table~\ref{tab:cnc_ablations} summarizes that across these design ablations, default \name{} as presented consistently outperforms these alternative implementations.

\begin{table}[h]
\centering
\caption{Ablation over \name{} algorithmic components on Waterbirds. Default choices achieve highest worst-group and average accuracy.}
\vspace{0.25cm}
\label{tab:cnc_ablations}

\begin{tabular}{@{}lcccc@{}}
\toprule
Method         & \name{} (Default) & Projection Head & One-sided Contrasting & Train + Finetune \\ \midrule
Worst-group Accuracy (\%)   & \textbf{88.5 (0.3)}                        & 82.4 (1.8)                          & 85.2 (3.6)                                & 84.0 (1.7)                           \\
Average Accuracy (\%) & \textbf{90.9 (0.1)}                        & 88.7 (0.6)                          & 90.1 (1.6)                                & 87.7 (1.1)           \\ \bottomrule
\end{tabular}
\end{table}

\textbf{No projection head.} We incorporate a nonlinear projection head as is typical in prior contrastive learning works \citep{chen2020simple}, that maps the encoder output to lower-dimensional representations (from $2048$ to $128$ in our case). We then update the encoder layers and the projection head jointly by computing the contrastive loss on the projection head's output, still passing the encoder layer's direct outputs to the classifier to compute the cross-entropy loss. We note that using the projection head decreases worst-group accuracy substantially. We reason that as previously discussed, while using the projection head in prior work can allow the model to retain more information in its actual hidden layers \citep{chen2020simple}, in our case to remove dependencies on spurious attributes we actually want to encourage learning invariant representations when we model the differences between anchor and positive datapoints as due to spurious attributes.

\textbf{Two-sided contrastive batches.}  Instead of ``two-sided'' contrasting where we allow both sampled anchors and positives to take on the anchor role, for each batch we only compute contrastive updates by comparing original positives and negatives with the original anchor. When keeping everything else the same, we find that just doing these one-sided comparisons also leads to a drop in performance for worst-group accuracy. This suggests that the increased number of comparisons where we swap the roles of anchors and positives introduces greater contrastive learning signal.  

\textbf{Joint training of encoder and classifier layers.} Instead of training the full model jointly, we first only train the encoder layers with the contrastive loss in \name{}, before freezing these layers and finetuning the classifier layers with the cross-entropy loss. With this implementation, we also obtain noticeable drop in performance. While we leave further analysis for the joint cross-entropy and contrastive optimization for future work, one conjecture is that the cross-entropy loss may aid in learning separable representations while also training the full model to keep the average error small. 

This also follows prior work, where updating the entire model and finetuning all model parameters instead of freezing the encoder layers leads to higher accuracy \citep{chen2020simple}. However, we found that with an initial encoder-only training stage, if we did not freeze the trained layers the fine-tuning on a dataset with spurious correlations would ``revert'' the contrastive training, resulting in a large gap between worst-group and average error similar to ERM.

\textbf{Contrastive loss balancing hyperparameter.} We also ablate the balancing hyperparameter $\lambda$ of \name{} on CMNIST$^*$. In Table~\ref{tab:cnc_ablations_lambda} we find that \name{} consistently achieves high worst-group accuracy across a wide range of $\lambda \in [0.4, 0.9]$. For reference, the next best methods GEORGE and JTT obtain $76.4\%$ and $74.5\%$ worst-group accuracy.

\begin{table}[h]
\centering
\caption{Ablation over \name{} $\lambda$ parameter to balance cross-entropy and contrastive loss components on CMNIST$^*$. \name{} obtains high performance across a range of $\lambda$.}
\vspace{0.25cm}
\label{tab:cnc_ablations_lambda}
{
\tabcolsep=0.5cm
\begin{tabular}{@{}lccccc@{}}
\toprule
\name{} $\lambda$ & 0.2        & 0.4        & 0.6        & 0.8        & 0.9        \\ \midrule
Robust Accuracy        & 70.4 (2.9) & 74.2 (2.6) & 75.3 (1.7) & 77.4 (2.5) & 75.8 (1.2) \\
Average Accuracy       & 89.0 (0.1) & 88.0 (0.7) & 88.3 (0.6) & 89.9 (0.4) & 88.4 (0.1) \\ \bottomrule
\end{tabular}
}
\end{table}

\subsection{{Comparison to domain generalization representation learning methods}}\label{appendix:dg_comparison}
While our main results in Table~\ref{tab:main_results} compare against methods designed to tackle the spurious correlations setting presented in Section~\ref{sec:results-main}, we also note that \name{} bears some similarity to methods proposed for unsupervised domain adaptation (UDA). At a high level, a popular approach is to learn similar representations for datapoints with the same class but sampled from different domains, e.g., via adversarial training to prevent another model from classifying representations' source domains correctly \citep{ganin2016domain}, or minimizing representation differences via metrics such as \textit{maximum mean discrepancy} (MMD) \citep{li2018domain}, to generalize to a desired target domain. While UDA carries distinct problem settings and assumptions from our spurious correlations setting (c.f. Appendix~\ref{sec:da_related_work}), we see if UDA methods that also try to optimize a model's representations can train models robust to spurious correlations, and compare their performance with \name{}. We first explain our protocol for fair evaluation. We then discuss results reported in Table~\ref{tab:da_dg_comparison}. 

We carry out our evaluation with domain-adversarial neural networks (DANN) \cite{ganin2016domain}, a seminal DG method that aims to learn aligned representations across two domains. To do so, DANN jointly trains a model to classify samples from a ``source'' domain while preventing a separate ``domain classifier'' module from correctly classifying the domain for datapoints sampled from both domains. For fair comparison, we use the same ResNet-50 backbone as in \name{}, and make several adjustments to the typical DANN and UDA procedure: 
\begin{enumerate}
    \item While UDA assumes that the data is organized into ``source'' and ``target'' domains, we do not have domain labels. We thus infer domains using the predictions of an initial ERM model as in \name{}.
    \item The notion of a domain may also be ambiguous with respect to the groups defined in Section~\ref{sec:preliminaries}. For example, domains may be defined by spurious attributes (e.g., for the Waterbirds dataset, we may consider the ``water background'' domain and the ``land background'' domain). Domains may alternatively be defined by whether samples carry dominant spurious correlations or not (e.g., the ``majority group'' domain and the ``minority group'' domain). We train and evaluate separate DANN models for both interpretations. We infer the former by the predicted class of the initial ERM model. We infer the latter by whether the initial ERM model is correct or not. 
    \item Finally, UDA aims to train with a class-labeled ``source'' domain and an unlabeled ``target'' domain such that a model performs well on unseen samples from the specified ``target'' domain \citep{ganin2016domain}. However, our benchmarks have class labels for \textit{all} training points, and do not have a notion of ``source'' and ``target'' domains (we aim to obtain high worst-group accuracy, which could fall under any domain). We thus assume access to labels for all domains. During training, the goal for our DANN models is to correctly classify samples from both domains, while learning representations such that a jointly trained domain classifier module cannot determine the samples' domains from their representations alone. At test-time, we evaluate the DANN model on the entire test set for each benchmark, and report the worst-group and average accuracies. 
\end{enumerate} 

In Table~\ref{tab:da_dg_comparison}, we report the worst-group and average accuracies of DANN on the Waterbirds and CelebA datasets across three seeds along with the \name{} results. Our results suggest that the domain alignment in DANN is not sufficient to improve worst-group accuracy. We hypothesize this is due to adversarial training with the domain classifier aligning representations without regard to different classes within each domain. Due to the propensity of samples exhibiting spurious correlations, DANN models may thus still learn to rely on these correlations.  

\begin{table}[H]

\centering

{
\caption{\name{} achieves higher worst-group and average accuracies on spuriously correlated benchmarks than DANN, a prior representation alignment method designed for domain adaptation}
\vspace{0.25cm}
\label{tab:da_dg_comparison}
\begin{tabular}{@{}lcccc@{}}

\toprule
Method                                       & \multicolumn{2}{c}{Waterbirds} & \multicolumn{2}{c}{CelebA} \\
Accuracy (\%)                                & Worst-group    & Average       & Worst-group  & Average     \\ \midrule
DANN (domains by spurious attribute)         & 37.4 (3.8)     & 87.6 (2.2)    & 28.1 (3.1)   & 94.6 (0.3)  \\
DANN (domains by majority vs minority group) & 67.3 (0.8)     & 83.6 (0.2)    & 47.2 (3.1)   & 88.7 (1.8)  \\
\name{}                     & \textbf{88.5 (0.3)} & \textbf{90.9 (0.1)} & \textbf{88.8 (0.9)} & \textbf{89.9 (0.5)}  \\ \bottomrule
\end{tabular}
}
\end{table}

\section{Further related work discussion}\label{sec:extended_related_work}
We provide additional discussion of related work and connections to our work below.

\subsection{Improving robustness to spurious correlations}\label{app:related_work_robustness_spurious}
Our core objective is to improve model robustness to group or subpopulation distribution shifts that arise from the presence of spurious correlations, specifically for classification tasks. Because these learnable correlations hold for some but not all samples in a dataset, standard training with ERM may result in highly variable performance: a model that classifies datapoints based on spurious correlations does well for some subsets or ``groups'' of the data but not others. To improve model robustness and avoid learning spurious correlations, prior work introduces the goal to maximize worst-group accuracy \citep{sagawa2019distributionally}. Related works broadly fall under two categories: 

\textbf{Improving robustness with group information.} If information such as spurious attribute labels is provided, one can divide the data into explicit groups as defined in Sec.~\ref{sec:preliminaries}, and then train to directly minimize the worst group-level error among these groups. This is done in group DRO (GDRO) \citep{sagawa2019distributionally}, where the authors propose an online training algorithm that focuses training updates over datapoints from higher-loss groups. \citet{goel2020model} also adopt this approach with their method CycleGAN Augmented Model Patching (CAMEL). However, similar to our motivation, they argue that a stronger modeling goal should be placed on preventing a model from learning group-specific features. Their approach involves first training a CycleCAN \citep{zhu2017unpaired} to learn the data transformations from datapoints in one group to another that share the same class label. They then apply these transformations as data augmentations to different samples, intuitively generating new versions of the original samples that take on group-specific features. Finally they train a new model with a consistency regularization objective to learn invariant features between transformed samples and their sources. Unlike their consistency loss, we accomplish a similar objective to learn group-invariant features with contrastive learning. Our first training stage is also less expensive. Instead of training a CycleGAN and then using it to augment datapoints, we train a relatively simple standard ERM classification model, sometimes with only a few number of epochs, and use its predictions to identify pairs of datapoints to serve a similar purpose. Finally, unlike both CAMEL and GDRO, we do not require spurious attribute or group labels for each training datapoints. We can then apply \name{} in less restrictive settings where such information is not known.

Related to GDRO are methods that aim to optimize a ``Pareto-fair'' objective, more general than simply the worst-case group performance. Notable examples are the works of \citet{balashankar2019what} and \citet{martinez2020minimax}. However, these approaches similarly do not directly optimize for good representation alignment, which is a focus of our work.

\textbf{Improving robustness without training group information.} More similar to our approach are methods that do not assume group information at training time, and only require validation set spurious attribute labels for fine-tuning. As validation sets are typically much smaller in size than training sets, 
an advantage of \name{} and comparable methods is that we can improve the accessibility of robust training methods to a wider set of problems. One popular line of work is distributionally robust optimization (DRO), which trains models to minimize the worst loss within a ball centered around the observed distribution \citep{ben2013robust, wiesemann2014distributionally, duchi2019variance, NEURIPS2020_64986d86, NEURIPS2020_0b6ace9e, oren2019distributionally}.  However, prior work has shown that these approaches may be too pessimistic, optimizing not just for worst-group accuracy but worst possible accuracy within the distribution balls \citep{sagawa2019distributionally}, or too undirected, optimizing for too many subpopulations, e.g. by first upweighting minority points but then upweighting majority points in later stages of training \citep{liu2021just}. \citet{pezeshki2020gradient} instead suggest that \textit{gradient starvation} (GS), where neural networks only learn to capture statistically dominant features in the data \citep{combes2018learning}, is the main culprit behind learning spurious correlations, and introduce a ``spectral decoupling'' regularizer to alleviate GS. However this does not directly prevent models from learning dependencies on spurious attributes. Similar to CAMEL, \cite{taghanaki2021robust} propose Contrastive Input Morphing (CIM), an image dataset-specific method that aims to learn input feature transformations that remove the effects of spurious or task-irrelevant attributes. They do so without group labels, training a transformation network with a triplet loss to transform input images such that a given transformed image's \emph{structural similarity metric} (based on luminance, contrast, and structure \citep{wang2003multiscale}) is more similar to a ``positive'' image from the same class than a ``negative'' image from a different class. They then train a classifier on top of these representations. Instead of pixel-level similarity metrics, \name{} enforces similarity in a neural network's hidden-layer representations, allowing \name{} to apply to non-image modalities. Additionally, we sample positives and negatives not just based on class label, but also the learned spurious correlations of an ERM model (via its trained predictions). We hypothesize that our sampling scheme, which intuitively provides "harder" positive and negative examples, allows \name{} to more strongly overcome spurious correlations.

Most similar to our approach are methods that first train an initial ERM model with the class labels as a way to identify data points belonging to minority groups, and subsequently train an additional model with greater emphasis on the estimated minority groups. \citet{NEURIPS2020_e0688d13} demonstrate that even when only trained on the class labels, neural networks learn feature representations that can be clustered into groups of data exhibiting different spurious attributes. They use the resulting cluster labels as estimated group labels before running GDRO on these estimated groups. Meanwhile, \citet{NEURIPS2020_eddc3427} train a pair of models, where one model minimizes a generalized cross-entropy loss \citep{NEURIPS2018_f2925f97}, such that the datapoints this model classifies incorrectly largely correspond to those in the minority group. They then train the other model on the same data but upweight the minority-group-estimated points. While they interweave training of the biased and robust model, \citet{liu2021just} instead train one model first with a shortened training time (but the standard cross-entropy objective), and show that then upsampling the incorrect data points and training another model with ERM can yield higher worst-group accuracy. \citet{creager2021environment} propose Environment Inference for Invariant Learning (EIIL), which first trains an ERM model, and then softly assign the training data into groups
under which the initial trained ERM model would maximally violate the invariant risk minimization (IRM) objective. In particular, the IRM objective is maximally satisfied if a model's optimal classifier is the same across groups \citep{arjovsky2019invariant}, and EIIL groups are inferred such that the initial ERM model's representations exhibit maximum variance within each group. \citet{creager2021environment} then runs GDRO with these groups.
Finally, \citet{nagarajan2020understanding} provides a theoretical understanding of how ERM picks up spurious features under data set imbalance.
They consider a setting involve a single spurious feature that is correlated with the class label and analyze the max-margin classifier in the presence of this spurious feature.

In our work, we demonstrate that the ERM model's predictions can be leveraged to not only estimate groups and train a new model with supervised learning but with different weightings. Instead, we can specifically identify pairs of points that a contrastive model can then learn invariant features between. Our core contribution comes from rethinking the objective with a contrastive loss that more directly reduces the model's ability to learning spurious correlations.

\subsection{Contrastive learning} Our method also uses contrastive learning, a simple yet powerful framework for both self-supervised \citep{chen2020simple, oord2018representation, tian2019contrastive, NEURIPS2020_5cd5058b, Sermanet2018TimeContrastiveNS, hassani2020contrastive, Robinson2020ContrastiveLW} and supervised \citep{NEURIPS2020_d89a66c7, gunel2020supervised} representation learning. The core idea is to learn data representations that maximize the similarity between a given input ``anchor'' and distinct different views of the same input (``positives''). Frequently this also involves \textit{contrasting} positives with ``negative'' data samples without any assumed relation to the anchor \citep{NEURIPS2019_ddf35421}. Core components then include some way to source multiple views, e.g. with data transformations \citep{chen2020simple}, and training objectives similar to noise contrastive estimation \citep{Gutmann2010NoisecontrastiveEA, Mnih2013LearningWE}.

An important component of contrastive learning is the method by which appropriate positives and negatives are gathered. For sampling positives, \citet{chen2020simple} show that certain data augmentations (e.g. crops and cutouts) may be more beneficial than others (e.g. Gaussian noise and Sobel filtering) when generating anchors and positives for unsupervised contrastive learning. \cite{von2021self} theoretically study how data augmentations help contrastive models learn core content attributes which are invariant to different observed ``style changes''. They propose a latent variable model for self-supervised learning. \citet{tian2020makes} further study what makes good views for contrastive learning. They propose an ``InfoMin principle'', where anchors and positives should share the least information necessary for the contrastive model to do well on the downstream task. For sampling negatives, \citet{Robinson2020ContrastiveLW} show that contrastive learning also benefits from using ``hard'' negatives, which (1) are actually a different class from the anchor (which they approximate in the unsupervised setting) and (2) embed closest to the anchor under the encoder's current data representation. Both of these approaches capture the principle that if positives are always too similar to the anchor and negatives are always too different, then contrastive learning may be inefficient at learning generalizable representations of the underlying classes. 

In our work, we incorporate this principle by sampling data points with the same class label but different ERM predictions\textendash presumably because of spurious attribute differences\textendash as anchor and positive views, while sampling negatives from data points with different class labels but the same ERM prediction as the anchor. The anchors and positives are different enough that a trained ERM model predicted them differently, while the anchors and negatives are similar enough that the trained ERM model predicted them the same. Contrasting the above then allows us to exploit both ``hard'' positive and negative criteria for our downstream classification task. In Section~\ref{sec:results-ablation}, we show that removing this ERM-guided sampling (i.e. only sampling positives and negatives based on class information), as well as trying different negative sampling procedures, leads to substantially lower worst-group accuracy with \name{}. 

One limitation of our current theoretical analysis regarding the alignment loss (cf. Section \ref{sec_analysis}) is that we require knowing the group labels to compute the RHS of equation \eqref{eq_result_align} (in particular, the alignment loss). An interesting question for future work is to provide a better theoretical understanding of the alignment induced by \name{} in the context of spurious correlations.

\subsection{Invariant learning}\label{app:invariant_learning}

Our work also shares some similarities in motivation with Invariant Risk Minimization (IRM) \citep{arjovsky2019invariant}, Predictive Group Invariance (PGI) \citep{Ahmed2021SystematicGW}, and other related works in domain-invariant learning \citep{krueger2020out, parascandolo2020learning, ahuja2020invariant, creager2021environment}. These methods aim to train models that learn a single invariant representation that is consistently optimal (e.g. with respect to classifying data) across different domains or environments. These environments can be thought of as data groups, and while traditionally methods such as IRM require that environment labels are known, recent approaches such as Environment Inference for Invariant Learning (EIIL) \citep{creager2021environment} and Predictive Group Invariance (PGI) \citep{Ahmed2021SystematicGW} similarly aim to infer environments with an initial ERM model. As discussed in Appendix~\ref{app:related_work_robustness_spurious}, in EIIL, \citet{creager2021environment} next train a more robust model with an invariant learning objective, similarly selecting models based on the worst-group error on the validation set. However, they train this model using IRM for Colored MNIST and GDRO for Waterbirds and CivilComments-WILDS, using the inferred environments as group labels. PGI uses EIIL to infer environments, but trains a more robust model by minimizing the KL divergence of the predicted probabilities for samples in the same class, but different groups, using the inferred environments as group labels. Thus, these approaches may similarly seek to learn similar representations for samples across groups, but do so via GDRO's reweighting, or the outputs of the model's classification layer. In contrast, \name{} trains a more robust model by using contrastive learning to affect its representations more directly. In our main evaluation (c.f. Section~\ref{sec:results-main}) we show that \name{}'s proposed contrastive loss objective and hard sampling strategy lead to higher worst-group accuracy.

\subsection{{Learning representations for unsupervised domain adaptation}}\label{sec:da_related_work}
\name{}'s approach to improve model robustness via a model's hidden-layer representations bears similarity to some prior unsupervised domain adaptation (UDA) methods. In UDA, the goal is to use the source and target features, and the source labels, to transfer to the specific target domain. As introduced in Appendix~\ref{appendix:dg_comparison}, to generalize beyond a single domain, one promising direction is to learn similar representations for samples from different domains. However, UDA methods assume knowledge of training sample domains or spurious attributes, whereas \name{} and our other comparable methods do not. UDA methods also assume a fundamentally different data setup; the training data is divided into source and target domains, and only the source domain has labels. In our setting, we \textit{do} have class labels for all samples available during training, but \textit{do not} have any natural definition of source and target domains (and thus no training data domain labels). Applying UDA methods requires additional reintepretation of this setup.

Considering methods to learn desirable representations, one popular UDA approach is domain adversarial neural networks (DANN) \cite{ganin2016domain}. DANN accomplishes alignment by adversarially training a model’s encoder layers (the ``feature extractor'') to learn representations such that a separate domain classifier module cannot distinguish samples' domains from the learned representations. To preserve class information, they train a classifier module on top of the feature extractor jointly with a cross-entropy loss. \name{}'s process for aligning representations is more simple. We do not rely on training separate modules with conflicting objectives to accomplish alignment; instead the supervised contrastive loss with \name{}'s sampling procedure encourages learning representations that are both separable across classes and aligned within each class. We thus avoid additional training parameters and optimization issues associated with minimax-based adversarial training \citep{arjovsky2017towards}. Instead of relying on a the domain classifier's output, we can train single model to align representations by minimizing the cosine similarity between anchor and positive samples. 

\subsection{{Learning representations for domain generalization}}\label{sec:dg_related_work}

\name{} also bears some similarity to representation learning methods for domain generalization (DG), which also try to learn representations that generalize beyond individual groups. However, in contrast to the spurious correlations problem, DG settings often assume access to multiple ``source'' domains and knowledge (labels) of which group or domain each sample belongs to. They aim to generalize to a specific unseen ``target'' domain not present during training. Unlike the spurious correlations presented in our evaluated datasets, class distributions within domains are also not canonically as skewed in standard benchmarks such as Office \citep{saenko2010adapting} and VLCS \citep{fang2013unbiased}. The distribution shift presented in these benchmarks is also distinct from the shift encountered with spurious correlations benchmarks. With spuriously correlated data, a model may learn to rely on spurious features themselves to obtain high average classification accuracy on the training set (e.g., solely relying on background features to classify bird type). At test time, the model may make confident predictions solely based on the absence or presence of these spurious features. In contrast, in DG (and UDA) settings, the domain-specific features themselves do not correlate with  classes. A model may learn to rely on domain-specific ``style'' features \textit{in conjunction} with other ``content-based'' features to classify images during training. The absence of the former style features in new domains may decrease performance, but not because the model has latched on specific spurious features to dictate its predictions. The two tasks thus differ via different model behaviors and failure modes. 

Methods to tackle DG via representation alignment include the invariant learning methods discussed in Appendix~\ref{app:invariant_learning}. Additionally, \cite{li2018domain} propose \textit{maximum mean discrepancy} MMD adversarial autoencoder (MMD-AAE). To align representations between domains, MMD-AAE (1) trains an autoencoder, (2) uses MMD maximum applied at the bottleneck hidden-layer to match representations across domains, and (3) applies an adversarial discriminator network to match these representations with a Laplace distribution (to encourage more sparse hidden representations). This matching is also not conditioned on the sample classes; an additional classifier head is applied to preserve class-specific information. \name{} is also simpler than MMD-AAE, only using the normalized dot product of a single classifier’s last hidden-layer representations. Via contrastive learning, \name{} also critically also aims to only align representations with the same class but different ERM-inferred groups, while pushing apart samples with different classes but the same ERM-inferred groups. By paying attention to the classes, we directly encourage a model to ignore group-specific information which confused the initial ERM model but that does not discriminate between classes.

\section{Additional experimental details}\label{appendix:experimental_details}
We first further describe our evaluation benchmarks in Appendix~\ref{appendix:dataset_details}. We next provide further details on how we calculate the reported metrics and the experimental hyperparameters of our main results in Appendix~\ref{appendix:dataset_details}. For all methods, following prior work \citep{liu2021just, NEURIPS2020_e0688d13, NEURIPS2020_eddc3427, sagawa2019distributionally, creager2021environment} we report the test set worst-group and average accuracies from models selected through hyperparameter tuning for the best validation set worst-group accuracy. While different methods have different numbers of tunable hyperparameters, we try to keep the number of validation queries as close as possible while tuning for fair comparison.

\subsection{Dataset details}\label{appendix:dataset_details}

\textbf{Colored MNIST (\cmnist).} We evaluate with a version of the Colored MNIST dataset proposed in \citet{arjovsky2019invariant}. The goal is to classify MNIST digits belonging to one of $5$ classes $\mathcal{Y} = \{$(0, 1), (2, 3), (4, 5), (6, 7), (8, 9)$\}$, and treat color as the spurious attribute. In the training data, we color $p_\text{corr}$ of each class's datapoints with an associated color $a$, and color the rest randomly. If $p_{\text{corr}}$ is high, trained ERM models fail to classify digits that are not the associated color. We pick $a$ from uniformly interspersed intervals of the \texttt{hsv} colormap, e.g. $0$ and $1$ digits may be spurious correlated with the color red (\texttt{\#ff0000}), while $8$ and $9$ digits may be spuriously correlated with purple (\texttt{\#ff0018}). The full set of colors in class order are $\mathcal{A} = \{\texttt{\#ff0000}, \texttt{\#85ff00}, \texttt{\#00fff3}, \texttt{\#6e00ff}, \texttt{\#ff0018}\}$ (see Fig.~\ref{fig:erm_motivation}). For validation and test data, we color each datapoint randomly with a color $a \in \mathcal{A}$. We use the default test set from MNIST, and allocate 80\%-20\% of the default MNIST training set to the training and validation sets. For main results, we set $p_\text{corr} = 0.995$.

\textbf{Waterbirds.} We evaluate with the Waterbirds dataset, which was introduced as a standard spurious correlations benchmark in \citet{sagawa2019distributionally}. In this dataset, masked out images of birds from the CUB dataset \citep{wah2011caltech} are pasted on backgrounds from the Places dataset \citep{zhou2017places}. Bird images are labeled either as waterbirds or landbirds; background either depicts water or land. From CUB, waterbirds consist of seabirds (ablatross, auklet, cormorant, frigatebird, fulmar, gull, jaeger, kittiwake, pelican, puffin, tern) and waterfowl (gadwell, grebe, mallard, merganser, guillemot, Pacific loon). All other birds are landbirds. From Places, water backgrounds consist of ocean and natural lake classes, while land backgrounds consist of bamboo forest and broadleaf forest classes.

The goal is to classify the foreground bird as $\mathcal{Y} = \{$waterbird, landbird$\}$, where there is spurious background attribute $\mathcal{A} = \{$water background, land background$\}$. We use the default training, validation, and test splits \citep{sagawa2019distributionally}, where in the training data 95\% of waterbirds appear with water backgrounds and 95\% of landbirds appear with land backgrounds. Trained ERM models then have trouble classifying waterbirds with land backgrounds and landbirds with water backgrounds. For validation and test sets, water and land backgrounds are evenly split among landbirds and waterbirds. 

\textbf{CelebA.} We evaluate with the CelebA spurious correlations benchmark introduced in \citet{sagawa2019distributionally}. The goal is to classify celebrities' hair color $\mathcal{Y} = \{$blond, not blond$\}$, which is spuriously correlated with the celebrity's identified gender $\mathcal{A} = \{$male, female$\}$. We use the same training, validation, test splits as in \citet{sagawa2019distributionally}. Only 6\% of blond celebrities are male; trained ERM models perform poorly on this group.

\textbf{CivilComments-WILDS.} We evaluate with the CivilComments-WILDS dataset from \citet{koh2021wilds}, derived from the Jigsaw dataset from \citet{borkan2019nuanced}. Each datapoint is a real online comment curated from the Civil Comments platform, a commenting plugin for independent news sites. For classes, each comment is labeled as either toxic or not toxic. For spurious attributes, each comment is also labeled with the demographic identities $\{$male, female, LGBTQ, Christian, Muslim, other religions, Black, White$\}$ mentioned; multiple identities may be mentioned per comment. 

The goal is to classify the comment $\mathcal{Y} = \{$toxic, not toxic$\}$. As in \citet{koh2021wilds}, we evaluate with $\mathcal{A} = \{$male, female, LGBTQ, Christian, Muslim, other religions, Black, White$\}$. There are then 16 total groups corresponding to (toxic, identity) and (not toxic, identity) for each identity. Groups may overlap; a datapoint falls in a group if it mentions the identity. We use the default data splits \citep{koh2021wilds}. In Table~\ref{tab:civilcomments_wilds_subgroups}, we list the percentage of toxic comments for each identity based on the groups. Trained ERM models in particular perform less well on the rarer toxic groups.

\begin{table}[h]
\centering
\caption{Percent of toxic comments for each identity in the CivilComments-WILDS training set.}
\vspace{0.25cm}
\label{tab:civilcomments_wilds_subgroups}
\begin{tabular}{@{}lcccccccc@{}}
\toprule
Identity & male & female & LGBTQ & Christian & Muslim & other religions & Black & White \\ \midrule
\% toxic & 14.9 & 13.7   & 26.9  & 9.1       & 22.4   & 15.3            & 31.4  & 28.0    \\ \bottomrule
\end{tabular}
\end{table}

\subsection{Implementation details}\label{appendix:experimental_details_methods}

\subsubsection{Reported metrics}\label{appendix:experimental_details_methods_metrics}

\textbf{Main results.}
For the \cmnist, Waterbirds, and CelebA datasets, we run \name{} with three different seeds, and report the average worst-group accuracy over these three trials in Table~\ref{tab:main_results}. As we use the same baselines and comparable methods as \citet{liu2021just}, we referenced their main results for the reported numbers, which did not have standard deviations or error bars reported. For CivilComments-WILDS, due to time and compute constraints we only reported one run.
{We note that \cmnist{} here is extremely challenging, as minority groups together only make up $0.5\%$ of the training set.} This severe imbalance explains the very poor worst-group performance of ERM (and other methods that fail to sufficiently remediate the issue).

\textbf{Estimated mutual information.} 
We give further details for calculating the representation metric introduced in Sec.~\ref{sec:erm_observations}. As a reminder, we report both alignment and estimated mutual information metrics to quantify how dependent a model's learned representations are on the class labels versus the spurious attributes, and compute both metrics on the representations $Z= \{f_\text{enc}(x)\}$ over all test set data points $x$. Then to supplement the alignment loss calculation in Sec.~\ref{sec:erm_observations}, we also estimate $I(Y; Z)$ and $I(A; Z)$, the mutual information between the model's data representations and the class labels and spurious attribute labels respectively.

To first estimate mutual information with $Y$, we first approximate $p(y \mid z)$ by fitting a multinomial logistic regression model over all representations $Z$ to classify $y$. With the empirical class label distribution $p(y)$, we compute:
\begin{equation}
    \hat{I}(Y; Z) =
    \frac{1}{|Z|} \sum_{z \in Z} \sum_{y \in Y} p(y \mid z) \log \frac{p(y \mid z)}{p(y)}
    \label{eq:estimated_mi}
\end{equation}

We do the same but substitute the spurious attributes $a$ for $y$ to compute $\hat{I}(A; Z)$.

\subsubsection{Stage 1 ERM training details} \label{appendix:experimental_details_methods_erm_details}
We describe the model selection criterion, architecture, and training hyperparameters for the initial ERM model in our method. To select this model, recall that we first train an ERM model to predict the class labels, as the model may also learn dependencies on the spurious attributes. Because we then use the model's predictions on the training data to infer samples with different spurious attribute values but the same class label, we prefer an initial ERM model that better learns this spurious dependency, and importantly also does not overfit to the training data. Inspired by the results in prior work \citep{NEURIPS2020_e0688d13, liu2021just}, we then explored using either a standard ERM model, one with high $\ell$-2 regularization (\texttt{weight decay = 1}), or one only trained on a few number of epochs. To select among these, because the validation data has both class labels and spurious attributes, we choose the model with the largest gap between worst-group and average accuracy on the validation set. \MZ{For fair comparison to JTT, we use the same batch size, learning rate, momentum, optimizer, default weight decay, and number of epochs as reported in \cite{liu2021just} to obtain these models}. We detail the ERM architecture and hyperparameters for each dataset below:

\textbf{Colored MNIST.} We use the LeNet-5 CNN architecture in the \texttt{pytorch} image classification tutorial. We train with SGD, few epochs $E = 5$, SGD, learning rate 1e-3, batch size 32, default weight decay 5e-4, and momentum 0.9.

\textbf{Waterbirds.} We use the \texttt{torchvision} implementation of ResNet-50 with pretrained weights from ImageNet as in \citet{sagawa2019distributionally}. Also as in \citep{sagawa2019distributionally}, we train with SGD, default epochs $E = 300$, learning rate 1e-3, batch size 128, and momentum 0.9. However we use high weight decay $1.0$.

\textbf{CelebA.} We also use the \texttt{torchvision} ImageNet-pretrained ResNet-50 and default hyperparameters from \citet{sagawa2019distributionally} but with high weight decay: we train with SGD, default epochs $E = 50$, learning rate 1e-4, batch size 128, momentum $0.9$, and high weight decay $0.1$.

\textbf{CivilComments-WILDS.} We use the HuggingFace (\texttt{pytorch-transformers}) implementation of BERT with pretrained weights and number of tokens capped at 300 as in \citet{koh2021wilds}. As in \citet{liu2021just}, with other hyperparameters set to their defaults \citep{koh2021wilds} we tune between using the AdamW optimizer with learning rate 1e-5 and SGD with learning rate 1e-5, momentum 0.9, and the PyTorch \texttt{ReduceLROnPlateau} learning rate scheduler. Based on our criterion, we use SGD, few number of epochs $E = 2$, learning rate 1e-5, batch size 16, default weight decay 1e-2, and momentum $0.9$.

\subsubsection{Contrastive batch sampling details} \label{appendix:experimental_details_methods_contrastive_batch_details}
We provide further details related to collecting predictions from the trained ERM models, 
and the number of positives and negatives that determine the contrastive batch size. 

\textbf{ERM model prediction.} To collect trained ERM model predictions on the training data, we explored two approaches: (1) using the actual predictions, i.e. the argmax for each classifier layer output vector, and (2) clustering the representations, or the last hidden-layer outputs, and assigning a cluster-specific label to each data point in one cluster. This latter approach is inspired by \citet{NEURIPS2020_e0688d13}, and we similarly note that ERM models trained to predict class labels in spuriously correlated data may learn data representations that are clusterable by spurious attribute. As a viable alternative to collecting the ``actual'' predictions of the trained ERM model on the training data, with $C$ classes, we can then cluster these representations into $C$ clusters, assign the same class label only to each data point in the same cluster, and choose the label-cluster assignment that leads to the highest accuracy on the training data. We also follow their procedure to first apply UMAP dimensionality reduction to 2 UMAP components, before clustering with K-means or GMM \citep{NEURIPS2020_e0688d13}. To choose between all approaches, we selected the procedure that lead to highest worst-group accuracy on the validation data after the second-stage of training. While this cluster-based prediction approach was chosen as a computationally efficient heuristic, we found that in practice it either lead to comparable or better final worst-group accuracy on the validation set. To better understand this, as a preliminary result we found that when visualizing the validation set predictions with the Waterbirds dataset, the cluster-based predictions captured the actual spurious attributes better than the classifier layer predictions (Fig.~\ref{fig:erm_prediction_comparison}). We defer additional discussion to \citet{NEURIPS2020_e0688d13} and leave further analysis to future work.

\begin{figure}[h]
  \vspace{-0.25cm}
  \centering
  \includegraphics[width=1\textwidth]{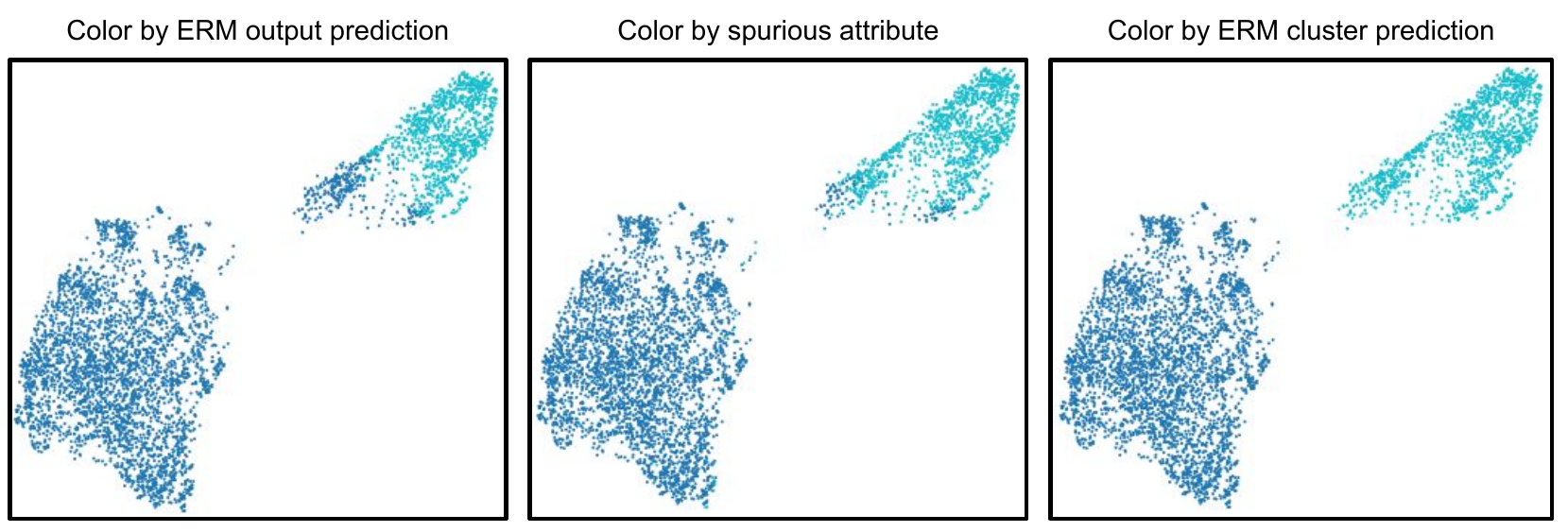}
  \caption{\small{UMAP visualization of ERM data representations for the Waterbirds training data. We visualize the last hidden layer outputs for a trained ERM ResNet-50 model given training samples from Waterbirds, coloring by either the ERM model's ``standard'' predictions, the actual spurious attribute values (included here just for analysis), and predictions computed by clustering the representations as described above. Clustering-based predictions more closely align with the actual spurious attributes than the ERM model outputs.}}
  \label{fig:erm_prediction_comparison}
\end{figure}

\textbf{Number of positives and negatives per batch.} One additional difference between our work and prior contrastive learning methods \citep{chen2020simple, NEURIPS2020_d89a66c7} is that we specifically construct our contrastive batches by sampling anchors, positives, and negatives first. This is different from the standard procedure of randomly dividing the training data into batches first, and then assigning the anchor, positive, and negative roles to each datapoint in a given batch. As a result, we introduce the number of positives $M$ and the number of negatives $N$ as two hyperparameters that primarily influence the size of each contrastive batch (with number of additional anchors and negatives also following $M$ and $N$ with two-sided batches). To maximize the number of positive and negative comparisons, as a default we set $M$ and $N$ to be the maximum number of positives and negatives that fit the sampling criteria specified under Algorithm~\ref{algo:two_sided_contrastive_batch} that also can fit in memory. In Appendix~\ref{appendix:experimental_details_methods_cnc_details}, for each dataset we detail the ERM prediction method and number of positives and negatives sampled in each batch.

\subsubsection{Stage 2 contrastive model training details} \label{appendix:experimental_details_methods_cnc_details}

In this section we describe the model architectures and training hyperparameters used for training the second model of our procedure, corresponding the reported worst-group and average test set results in Table~\ref{tab:main_results}. In this second stage, we train a new model with the same architecture as the initial ERM model, but now with a contrastive loss and batches sampled based on the initial ERM model's predictions. We report test set worst-group and average accuracies from models selected with hyperparameter tuning and early stopping based on the highest validation set worst-group accuracy. For all datasets, we sample contrastive batches using the clustering-based predictions of the initial ERM model. Each batch size specified here is also a direct function of the number of positives and negatives: $2M + 2N$.

\textbf{Colored MNIST.} We train a LeNet-5 CNN. For \name{}, we use $M = 32$, $N = 32$, batch size $128$, temperature $\tau = 0.05$, contrastive weight $\lambda = 0.75$, SGD optimizer, learning rate 1e-3, momentum 0.9, and weight decay 1e-4. We train for $3$ epochs, and use gradient accumulation to update model parameters every $32$ batches. 

\textbf{Waterbirds.} We train a ResNet-50 CNN with pretrained ImageNet weights. For \name{}, we use $M = 17$, $N = 17$, batch size $68$, temperature $\tau = 0.1$, contrastive weight $\lambda = 0.75$, SGD optimizer, learning rate 1e-4, momentum 0.9, weight decay 1e-3. We train for $5$ epochs, and use gradient accumulation to update model parameters every $32$ batches. 

\textbf{CelebA.} We train a ResNet-50 CNN with pretrained ImageNet weights. For \name{}, we use $M = 64$, $N = 64$, batch size $256$, temperature $\tau = 0.05$, contrastive weight $\lambda = 0.75$, SGD optimizer, learning rate 1e-5, momentum 0.9, and weight decay 1e-1. We train for $15$ epochs, and use gradient accumulation to update model parameters every $32$ batches. 

\textbf{CivilComments-WILDS.} We train a BERT model with pretrained weights and max number of tokens $300$. For \name{}, we use $M = 16$, $N = 16$, batch size $64$, temperature $\tau = 0.1$, contrastive weight $\lambda = 0.75$, AdamW optimizer, learning rate 1e-4, weight decay 1e-2, and clipped gradient norms. We train for $10$ epochs, and use gradient accumulation to update weights every $128$ batches.  

\subsubsection{Comparison method training details}
As reported in the main results (Table~\ref{tab:main_results}) we compare \name{} with the ERM and Group DRO baselines, as well as robust training methods that do not require spurious attribute labels for the training data: \textsc{George} \citep{NEURIPS2020_64986d86}, Learning from Failure (LfF) \citep{NEURIPS2020_64986d86}, Predictive Group Invariance (PGI) \citep{Ahmed2021SystematicGW}, Contrastive Input Morphing (CIM) \citep{taghanaki2021robust}, Environment Inference for Invariant Learning (EIIL) \citep{creager2021environment}, and Just Train Twice (\textsc{Jtt}) \citep{liu2021just}. For each dataset, we use the same model architecture for all methods. For the Waterbirds, CelebA, and CivilComments-WILDS datasets, we report the worst-group and average accuracies reported in \citet{liu2021just} for ERM, LfF, and \textsc{Jtt}. For \textsc{George}, we report the accuracies reported in \citet{NEURIPS2020_e0688d13}. For CIM, we report results from Waterbirds and CelebA from \cite{taghanaki2021robust} using the CIM + variational information bottleneck implementation \cite{alemi2016deep}, which achieves the best worst-group performance in their results. For EIIL, we report results from Waterbirds and CivilComments-WILDS from \cite{creager2021environment}. For these hyperparameters, we defer to the original papers. For GDRO, we reproduce the results with the same optimal hyperparameters over three seeds. For PGI, we used our own implementation for all results, with details specified below. For Colored MNIST, we run implementations for \textsc{George}, CIM, EIIL, LfF, \textsc{Jtt}, and GDRO using code from authors. We include training details for our own implementations below:

\textbf{Colored MNIST (\cmnist).}
We run all methods for 20 epochs, reporting test set accuracies with early stopping. For \textsc{Jtt}, we train with SGD, learning rate 1e-3, momentum 0.9, weight decay 5e-4, batch size 32. We use the same initial ERM model as \name{}, with hyperparameters described in Appendix~\ref{appendix:experimental_details_methods_erm_details}. For upsampling we first tried constant factors $\{10, 100, 1000\}$. We also tried a resampling strategy where for all the datapoints with the same initial ERM model prediction, we upsample the incorrect points such that they equal the correct points in frequency, and found this worked the best. With $p_\text{corr} = 0.995$, this upsamples each incorrect point by roughly $1100$. We also use this approach for the results in Fig.~\ref{fig:representation_effect}. For GDRO we use the same training hyperparameters as \textsc{Jtt}, but without the upsampling and instead set group adjustment parameter $C = 0$. For LfF, we use the same hyperparameters as \textsc{Jtt}, but instead of upsampling gridsearched the $q$ parameter $\in \{0.1, 0.3, 0.5, 0.7, 0.9\}$, using $q = 0.7$. For \textsc{George} we train with SGD, learning rate 1e-3, momentum 0.9, weight decay 5e-4. For CIM, we use the CIM + VIB implementation. We train with SGD, learning rate 1e-3, weight decay 5e-4, $\beta$ parameter $10$, and $\lambda$ parameter 1e-5. For EIIL, for environment inference we use the same initial ERM model as \name{} and \textsc{Jtt}, and update the soft environment assignment distribution with Adam optimizer, learning rate 1e-3, and $10000$ steps. Following \cite{creager2021environment}'s own colored MNIST experiment, we train the second model with IRM, using learning rate 1e-2, weight decay 1e-3, penalty weight $100$, and penalty annealing parameter $80$.

\textbf{CelebA.}
We also tune EIIL for CelebA. We again use the same initial ERM model as \name{}, and update the soft environment assignment distribution with Adam optimizer, learning rate 1e-3, and $10000$ steps. We train the second model with GDRO, using SGD, $50$ epochs, learning rate 1e-5, batch size 128, weight decay $0.1$, and group adjustment parameter $3$. 

\textbf{PGI.} To compare against PGI, we tried two implementations. First, we followed the PGI algorithm to first infer environments via the same mechanism as in EIIL [2], and trained a second model with the PGI objective using standard shuffled minibatches (aiming to minimize the KL divergence for samples with the same class but different inferred environment labels per batch). However, despite ample hyperparameter tuning (trying loss weighting component $\lambda \in \{0.1, 0.5, 10,  100\}$, we could not get PGI to work well (on Waterbirds, we obtained $51.0 \pm 4.9\%$ worst-group accuracy and $79.6 \pm 2.6\%$ average accuracy). We hypothesize this is due to the strong spurious correlations in our datasets: \cite{Ahmed2021SystematicGW} only consider datasets where 20\% of the training samples do not exhibit a dominant correlation and fall under minority groups. Our evaluation benchmarks are more difficult due to stronger spurious correlations: in Waterbirds only $5\%$ of samples do not exhibit the dominant correlation, and in CelebA only $7\%$ of training samples lie in the smallest group.

We then tried a more balanced batch variation. Instead of using randomly shuffled minibatches, we use PGI environment inference labels to sample batches similarly to how \name{} uses the stage 1 ERM model predictions to sample batches. We construct batches by specifying the same number of “anchors”, “positives”, and “negatives” as in \name{}, and sample batches where anchors and positives are samples with the same class, but different inferred environments. Anchors and negatives are samples in the same inferred environment, but with different classes. We then train a second model with the PGI criterion using these modified batches.

In Section~\ref{appendix:hparam_sweep}, we include our sweeps for both the method-specific and general hyperparameters.

\textbf{Comparison limitations.}
One limitation of our comparison is that because for each dataset we sample new contrastive batches which could repeat certain datapoints, the number of total batches per epoch changes. For example, $50$ epochs training the second model in \name{} does not necessarily lead to the same total number of training batches as $50$ epochs training with ERM, even if they use the same batch size. However, we note that the numbers we compare against from \citet{liu2021just} are reported with early stopping. In this sense we are comparing the best possible worst-group accuracies obtained by the methods, not the highest worst-group accuracy achieved within a limited number of training batches. We also found that although in general the time to complete one epoch takes much longer with \name{}, \name{} requires fewer overall training epochs for all but the CivilComments-WILDS dataset to obtain the highest reported accuracy.

\subsubsection{Hyperparameter sweeps}\label{appendix:hparam_sweep}
To fairly compare with previous methods \citep{liu2021just, creager2021environment, taghanaki2021robust}, we use the same evaluation scheme (selecting models based on worst-group validation error), and sweep over a consistent number of hyperparameters, i.e. number of validation set queries. We set this number for \name{} to be a comparable number of queries that is reported in prior works. We break this down into method-specific (e.g. contrastive temperature in \name{}, upweighting factor in \textsc{Jtt}), and shared (e.g. learning rate) hyperparameter categories.

\textbf{Method-specific.}
For \name{}, we tune three method-specific hyperparameters: contrastive loss temperature (Eq.~\ref{eq:supcon_easy_anchor}), contrastive weight (Eq.~\ref{eq:full_loss}), and gradient accumulation steps values as in Table~\ref{tab:cnc_hparam_sweep}.
\begin{table}[H]
\centering
\caption{Method-specific hyperparameters for \name{}.}
\vspace{0.25cm}
\begin{tabular}{@{}ccc@{}}
\toprule
Hyperparameter                               & Dataset                    & Values            \\ \midrule
Temperature ($\tau$)                         & All                        & $\{0.05, 0.1\}$   \\ \midrule
Contrastive Weight ($\lambda$)               & All                        & $\{0.5, 0.75\}$   \\ \midrule
\multirow{2}{*}{Gradient Accumulation Steps} & CMNIST$^*$, Waterbirds, CelebA & $\{32, 64\}$ \\
                                             & CivilComments-WILDS        & $\{32, 64, 128\}$      \\ \bottomrule
\end{tabular}
\vspace{-0.25cm}
\label{tab:cnc_hparam_sweep}
\end{table}
For \textsc{Jtt}, the reported results and our CMNIST$^*$ implementation are tuned over the following hyperparameters in Table~\ref{tab:jtt_hparam_sweep}.
\begin{table}[H]
\centering
\vspace{-0.5cm}
\caption{Method-specific hyperparameters for \textsc{Jtt}.}
\vspace{0.25cm}
\begin{tabular}{@{}ccc@{}}
\toprule
Hyperparameter                           & Dataset                             & Values                                        \\ \midrule
\multirow{2}{*}{Stage 1 Training Epochs} & Waterbirds                          & $\{40, 50, 60\}$                              \\
                                         & CMNIST$^*$, CelebA, CivilComments-WILDS & $\{1, 2\}$                                    \\ \midrule
\multirow{3}{*}{Upweighting Factor}      & CMNIST$^*$                              & \multicolumn{1}{l}{$\{10, 100, 1000, 1100\}$} \\
                                         & Waterbirds, CelebA                  & $\{20, 50, 100\}$                             \\
                                         & CivilComments-WILDS                 & $\{4, 5, 6\}$                                 \\ \bottomrule
\end{tabular}
\label{tab:jtt_hparam_sweep}
\end{table}

For EIIL, our CMNIST$^*$ and CelebA implementations are tuned over hyperparameters reported in Table~\ref{tab:eiil_hparam_sweep}. \cite{creager2021environment} report that they allow up to 20 evaluations with different hyperparameters for Waterbirds and CivilComments-WILDS. When using GDRO as the second stage model, they also report using the same hyperparameters as the GDRO baseline for Waterbirds. We do the same for our evaluation on CelebA. This amounts to primarily tuning the first stage environment inference learning rate and number of updating steps for CMNIST$^*$ and CelebA, and the penalty annealing iterations and penalty weight for the IRM second stage model for CMNIST$^*$.

\begin{table}[H]
\centering
\caption{Method-specific hyperparameters for EIIL.}
\vspace{0.25cm}
\begin{tabular}{@{}ccc@{}}
\toprule
Hyperparameter                   & Dataset & Values                 \\ \midrule
Environment Inference Learning Rate                 & CMNIST$^*$, CelebA     & \{1e-1, 1e-2, 1e-3\}   \\ \midrule
Environment Inference Update Steps                  & CMNIST$^*$, CelebA     & \{10000, 20000\}       \\ \midrule
IRM Penalty Weight               & CMNIST$^*$  & \{0.1, 10, 1000, 1e5\} \\ \midrule
IRM Penalty Annealing Iterations & CMNIST$^*$  & \{10, 50, 80\}         \\ \midrule
GDRO Group Adjustment            & CelebA  & \{0, 2, 3\}            \\ \bottomrule
\end{tabular}
\label{tab:eiil_hparam_sweep}
\end{table}

For CIM, our CMNIST$^*$ implementation uses CIM + VIB, and is tuned over the $\beta$ VIB parameter \cite{alemi2016deep} and contrastive weighting parameter $\lambda$ for CIM in Table~\ref{tab:cim_hparam_sweep}. \cite{taghanaki2021robust} report tuning over a range of values within [1e-5, 1] for $\lambda$ on the CelebA and Waterbirds datasets.

\begin{table}[H]
\centering
\caption{Method-specific hyperparameters for CIM.}
\vspace{0.25cm}
\begin{tabular}{@{}ccc@{}}
\toprule
Hyperparameter & Dataset    & Values                   \\ \midrule
CIM $\lambda$  & CMNIST$^*$ & \{0.01, 0.05, 0.1\}      \\ \midrule
VIB $\beta$    & CMNIST$^*$ & \{1e-5, 1e-3, 1e-1, 10\} \\ \bottomrule
\end{tabular}
\label{tab:cim_hparam_sweep}
\end{table}

For PGI, we tune $\lambda$ with the same environment inference parameters as used in EIIL (Table~\ref{tab:eiil_hparam_sweep}). Fixing these parameters to infer environments, we tuned the $\lambda$ component for training the robust model across $\lambda \in \{0.1, 0.5, 10, 100\}$.

\textbf{Shared.} For all datasets, we use the same optimizer and momentum (if applicable) as reported in the \textsc{Jtt} paper. Table~\ref{tab:shared_hparam_sweep} contains the data-specific shared hyperparameter values tried.

\begin{table}[H]
\centering
\vspace{-0.5cm}
\caption{Shared hyperparameters}
\vspace{0.25cm}
\begin{tabular}{@{}ccccc@{}}
\toprule
              & CMNIST$^*$         & Waterbirds & CelebA & CivilComments-WILDS \\ \midrule
Learning Rate & \{1e-4, 1e-3, 1e-2\} &   \{1e-4, 1e-3\}         &  \{1e-5, 1e-4\}       &     \{1e-5, 1e-4\}                \\ \midrule
Weight Decay  & \{1e-4, 5e-4\}  &    \{1e-4, 1e-3\}        &  \{1e-2, 1e-1\}      &     \{1e-2\}                \\ \bottomrule
\end{tabular}
\label{tab:shared_hparam_sweep}
\end{table}

\subsection{\name{} compute resources and training time}
All experiments for CMNIST$^*$, Waterbirds, and CelebA were run on a machine with 14 CPU cores and a single NVIDIA Tesla P100 GPU. Experiments for CivilComments-WILDS were run on an Amazon EC2 instance with eight CPUs and one NVIDIA Tesla V100 GPU. 

Regarding runtime, one limitation with the current implementation of \textsc{CnC} is its comparatively longer training time compared to methods such as standard ERM. This is both a result of training an initial ERM model in the first stage, and training another model with contrastive learning in the second stage. In Table~\ref{tab:training time} we report both how long it takes to train the initial ERM model and long it takes to complete one contrastive training epoch on each dataset. We observe that while in some cases training the initial ERM model is negligible, especially if we employ training with only a few epochs to prevent memorization (for Colored MNIST it takes roughly two minutes to obtain a sufficient initial ERM model), it takes roughly 1.5 and 3 hours to train the high regularization initial models used for Waterbirds and CelebA. While these hurdles are shared by all methods that train an initial ERM model, we find that the second stage of \name{} occupies the bulk of training time. Prior work has shown that contrastive learning requires longer training times and converges more slowly than supervised learning \citep{chen2020simple}. We also observe this in our work. 

We note however that because we sample batches based on the ERM model's predictions, the contrastive training duration is limited by how many datapoints the initial ERM model predicts incorrectly. In moderately sized datasets with very few datapoints in minority groups, (e.g. Waterbirds, which has roughly 4794 training points and only 56 datapoints in its smallest group), the total time it takes to train \name{} is on par with ERM. Additionally, other methods such as additional hard negative mining \citep{Robinson2020ContrastiveLW} have been shown to improve the efficiency of contrastive learning, and we can incorporate these components to speed up training time as well.

\begin{table}[H]
\centering
\caption{\name{} Average total training time for first and second stages of \textsc{Cnc}}
\vspace{0.25cm}
\label{tab:training time}
\begin{tabular}{@{}lcccc@{}}
\toprule
Dataset                             & CMNIST$^*$ & Waterbirds & CelebA   & CivilComments-WILDS \\ \midrule
Stage 1 ERM train time              & 2 min.         & 1.5 hrs    & 3 hrs    & 3.1 hrs               \\
Stage 2 \name{} train time & 1.2 hrs        & 1.8 hrs    & 32.2 hrs &  37.6 hrs                   \\ \bottomrule
\end{tabular}
\end{table}

\newpage
\section{Visualization of learned data representations}
As in Fig.~\ref{fig:representation_visuals}, we visualize and compare the learned representations of test set samples from models trained with ERM, \textsc{Jtt}, and \name{} in Fig.~\ref{fig:representation_visuals_cmnist_wbirds_celeba}. Compared to ERM models, both \textsc{Jtt} and \name{} models learn representations that better depict dependencies on the class labels. However, especially with the Waterbirds and CelebA datasets, \name{} model representations more clearly depict dependencies only on the class label, as opposed to \textsc{Jtt} models which also show some organization by the spurious attribute still.

\begin{figure}[h]
  \centering
  \includegraphics[width=1\textwidth]{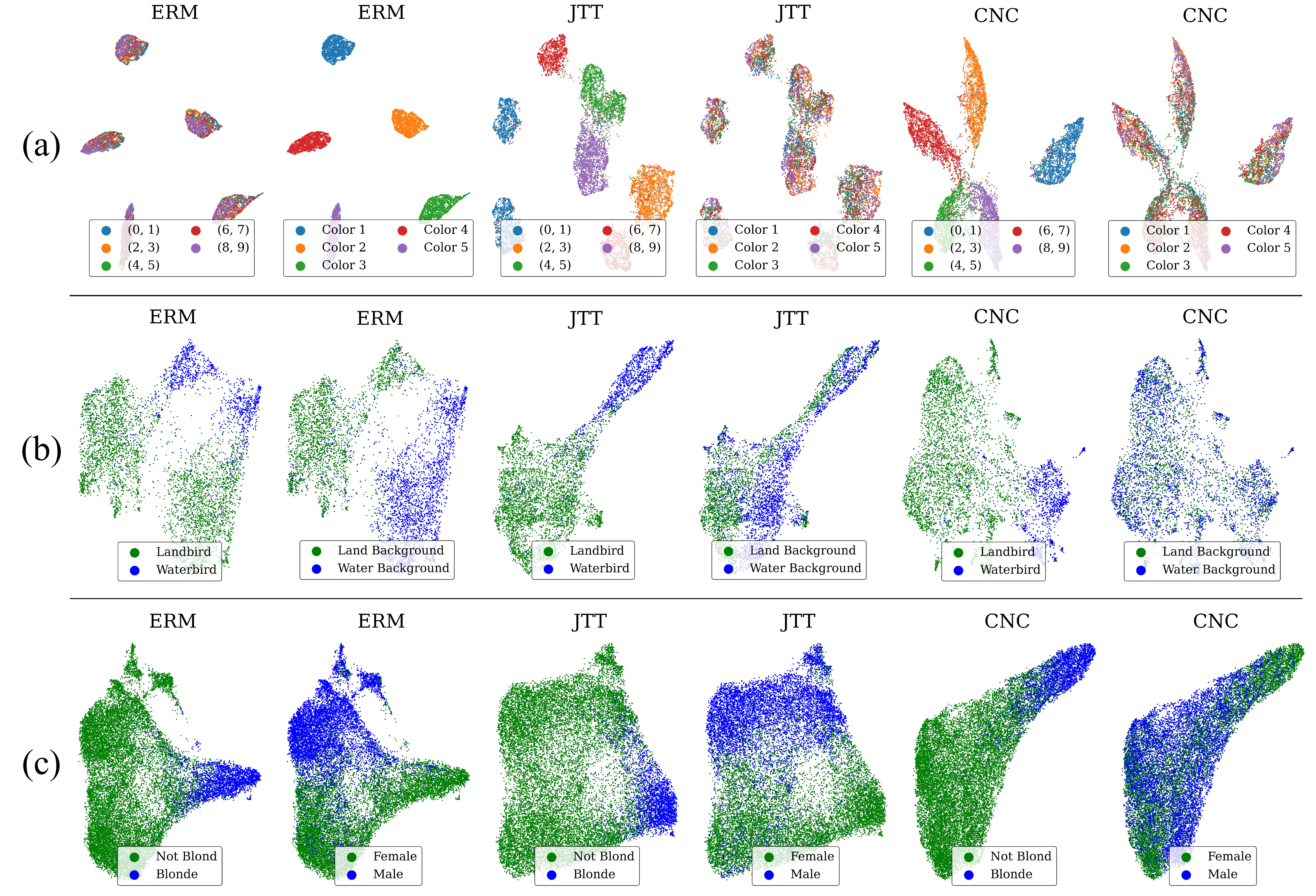}
  \caption{\small{UMAP visualizations of learned representations for Colored MNIST (a), Waterbirds (b), and CelebA (c). We color data points based on the class label (left) and spurious attribute (right). Most consistently across datasets, \name{} representations exhibit dependence and separability by the class label but not the spurious attribute, suggesting that they best learn features which only help classify class labels.}}
  \label{fig:representation_visuals_cmnist_wbirds_celeba}
\end{figure}

\section{Additional GradCAM visualizations}\label{appendix:gradcams}
On the next two pages, we include additional GradCAM visualizations depicting saliency maps for samples from each group in the Waterbirds and CelebA datasets. Warmer colors denote higher saliency, suggesting that the model considered these pixels more important in making the final classification as measured by gradient activations. For both datasets, we compare maps from models trained with ERM, the next most competitive method for worst-group accuracy \textsc{Jtt}, and \name{}. \name{} models most consistently measure highest saliency with pixels directly associated with class labels and not spurious attributes.

\newpage

\begin{figure}[h]
  \centering
  \includegraphics[width=0.9\textwidth]{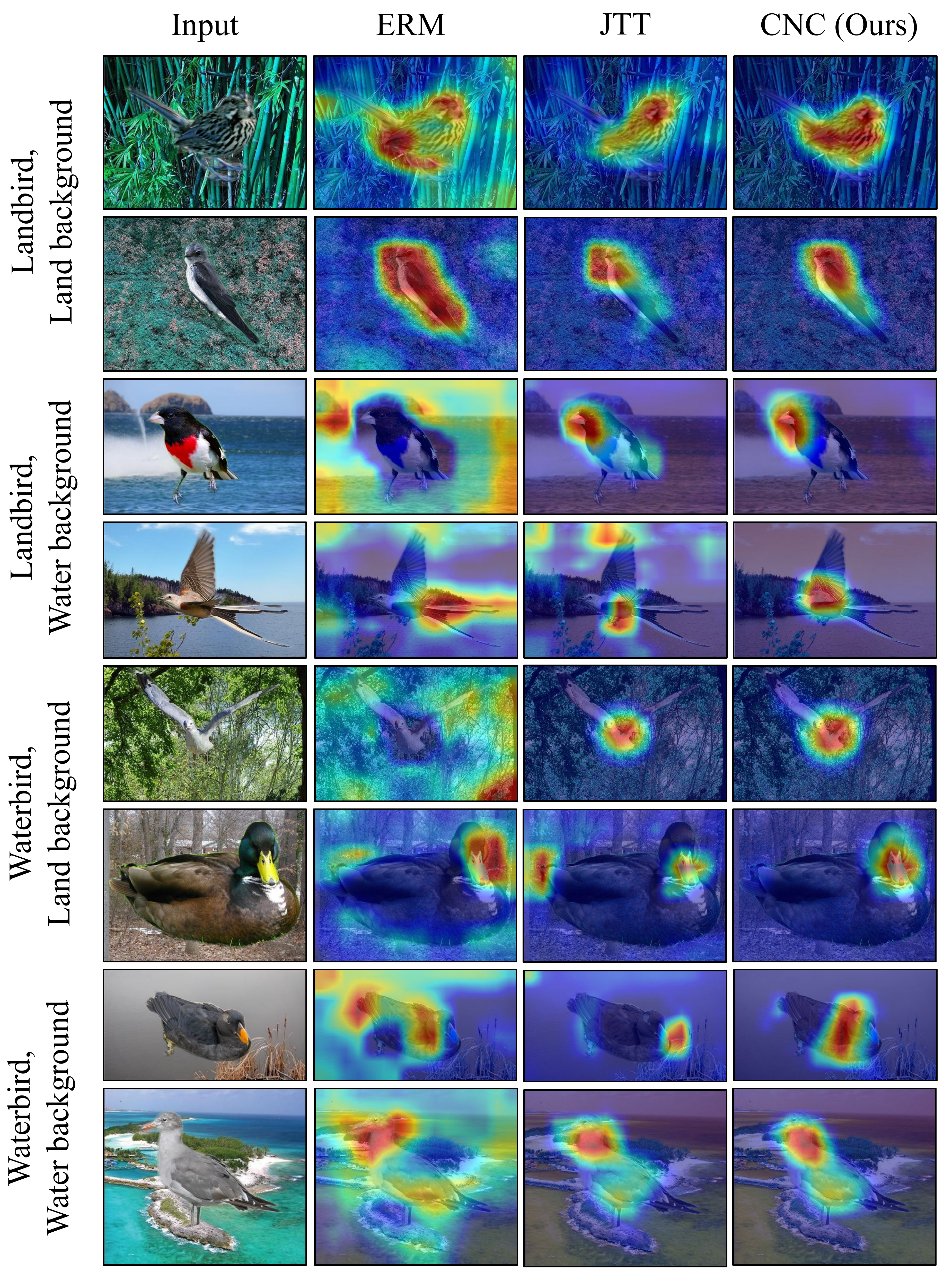}
  \caption{\small{Additional GradCAM visualizations for the Waterbirds dataset. We use GradCAM to visualize the ``salient'' observed features used to classify images by bird type for models trained with ERM, \textsc{Jtt}, and \textsc{CnC}. ERM models output higher salience for spurious background attribute pixels, sometimes almost exclusively. \textsc{Jtt} and \textsc{CnC} models correct for this, with \textsc{CnC} better exclusively focusing on bird pixels. }}
  \label{fig:more_waterbirds_gradcams}
\end{figure}

\newpage

\begin{figure}[ht]
\begin{adjustbox}{addcode={\begin{minipage}{\width}}{\caption{
      Additional GradCAM visualizations for the CelebA dataset. For models trained with ERM, \textsc{Jtt}, and \textsc{CnC}, we use GradCAM to visualize the ``salient'' observed features used to classify whether a celebrity has blond(e) hair. ERM models interesting also ``ignore'' the actual hair pixels in favor of other pixels, presumably associated with the spurious gender attribute. In contrast, GradCAMs for \textsc{Jtt} and \textsc{CnC} models usually depict higher salience for regions that at least include hair pixels. \textsc{CnC} models most consistently do so. 
      }\end{minipage}},rotate=90,center}
\includegraphics[width=1.4\textwidth]{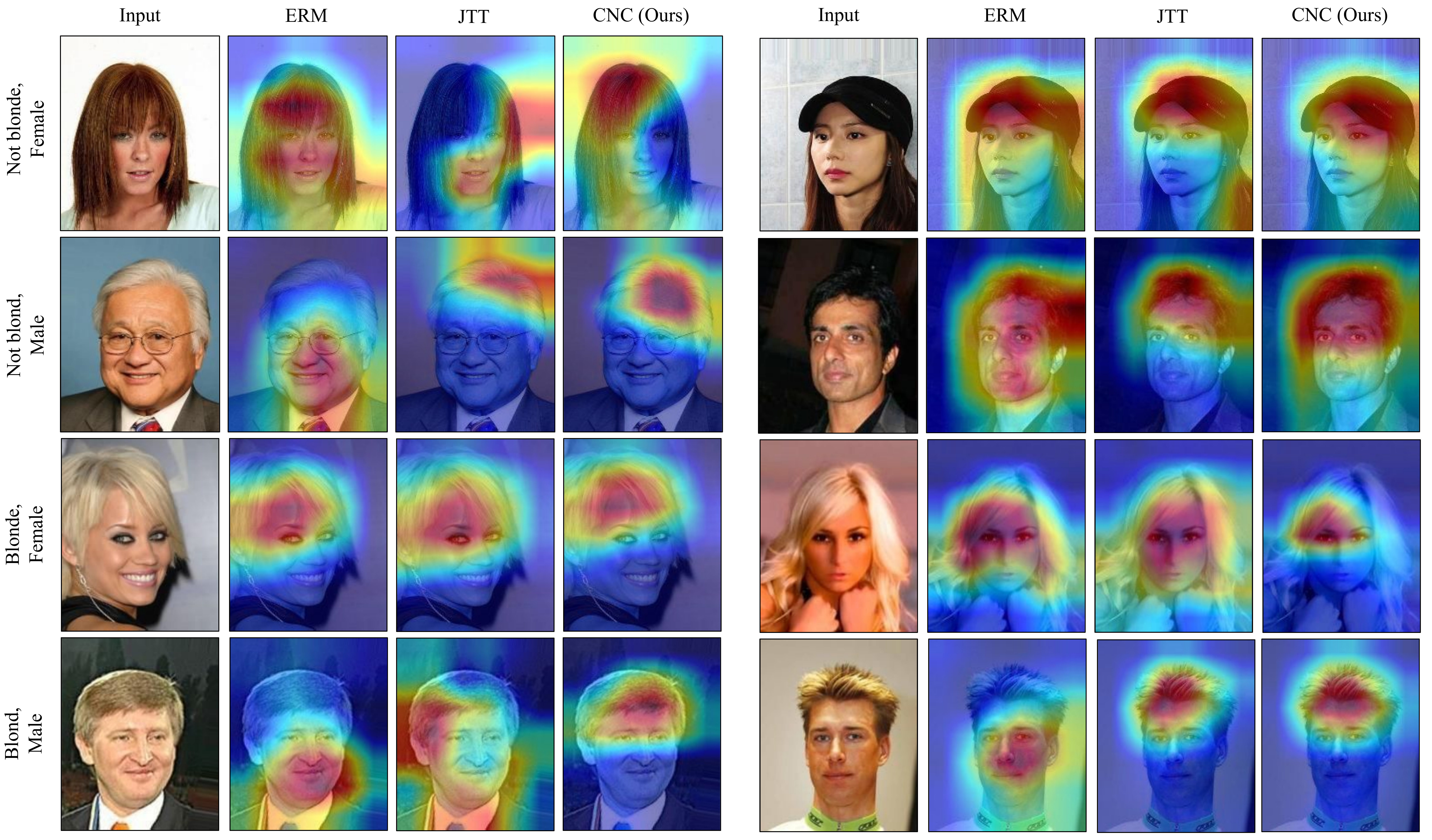}
  \end{adjustbox}
\end{figure}

%% file: proof.tex
\section{Omitted Proofs from Section \ref{sec_analysis}}\label{sec_proof_align}

We now prove that within any class, the gap between the worst-group error and the average error can be upper bounded by the alignment loss times the Lipschitz constant, plus another concentration error term.

\begin{proof}[Proof of Theorem \ref{prop_close}]
    Consider two arbitrary groups, denoted by $g_1 = (y, a_1)$ and $g_2 = (y, a_2)$, whose class labels are both $y\in\cY$, whose spurious attributes are $a_1\in\cA$ and $a_2\in \cA$ such that $a_1 \neq a_2$.
    Let $G_1$ and $G_2$ be the subset of training data that belong to groups $g_1$ and $g_2$, respectively.
    We note that both $G_1$ and $G_2$ are non-empty since we have assumed that (in Section \ref{sec:preliminaries}) there is at least one sample from each group in the training data set.
    Let $n_{g_1} = \abs{G_1}$ and $n_{g_2} = \abs{G_2}$ be the size of these two groups, respectively.
    Recall that $f_\text{enc}$ denotes the mapping of the encoder layers of the full neural network model $f_{\theta}$.
    Since the classification layer $f_{\text{cls}}$ is a linear layer, we have used $W$ to denote the weight matrix of this layer.
    Our definition of the cross-group alignment loss in equation \eqref{eq_cross_group}, denoted as $\hat{\cL}_{\alig}(f_{\theta}; y)$, implies that for $g_1$ and $g_2$,
    \begin{align}\label{eq_group_align}
        \frac{1}{n_{g_1}}\frac{1}{n_{g_2}} \sum_{(x, y, a_1)\in G_1}\sum_{(x', y, a_2)\in G_2} \norm{f_\text{enc}(x) - f_\text{enc}(x')}_2 \le \hat{\cL}_{\alig}(f_{\theta}; y).
    \end{align}
    Next, let $\exarg{(x, y, a_1) \sim \cP_{g_1}}{\cL_{\avg}(W f_{\enc}(x), y)}$ be the average loss conditioning on a data point being sampled from group $g_1$ (and similarly for group $g_2$).
    Let $\Delta(g_1, g_2)$ be the difference between the population average losses:
    \begin{align*}
        \Delta(g_1, g_2) = \bigabs{\exarg{(x, y, a_1)\sim \cP_{g_1}}{\cL_{\avg}(W f_\text{enc}(x), y} - \exarg{(x, y, a_2)\sim \cP_{g_2}}{\cL_{\avg}(W f_\text{enc}(x), y)}}.
    \end{align*}
    Recall that $\cG_y \subseteq \cG$ is the set of groups that have class label $y$.
    Since the loss $\ell(\cdot)$ is bounded above by some fixed constant $C_2$ according to our assumption, and is at least zero, by the Hoeffding's inequality, the following result holds with probability at least $1 - \delta$, for all $\abs{\cG_y}$  groups $g\in\cG_y$,
    \begin{align}\label{eq_wg_concentration}
        \bigabs{\exarg{(x,y,a)\sim\cP_{g}}{\cL_{\avg}(W f_\text{enc}(x), y)} - \frac{1}{n_{g}} \sum_{(x, y)\in (X, Y)} \ell(W f_\text{enc}(x), y)} \le C_2 \sqrt{\frac{2\log{(\abs{\cG_y}/\delta)}}{n_{g}}}.
    \end{align}
    Thus, with probability at least $1 - \delta$, the following holds for any $g_1$ and $g_2$ in class $y$ (but having different spurious attributes)
    \begin{align}
        \Delta(g_1, g_2) \le& \bigabs{\frac{1}{n_{g_1}}\sum_{(x, y, a_1)\in G_1} \cL_{\avg}(W f_\text{enc}(x), y) - \frac{1}{n_{g_2}}\sum_{(x', y, a_2)\in G_2} \cL_{\avg}(W f_\text{enc}(x'), y)} \label{eq_emp_delta} \\
        &+ C_2 \bigbrace{\sqrt{\frac{2\log(\abs{\cG_y}/\delta)}{n_{g_1}}} + \sqrt{\frac{2\log(\abs{\cG_y}/\delta)}{n_{g_2}}}}. \nonumber
    \end{align}
    Next, we focus on the RHS of equation \eqref{eq_emp_delta}.
    First, equation \eqref{eq_emp_delta} is also equal to the following:
    \begin{align*}
        \left|\frac{1}{n_{g_1}}\frac{1}{n_{g_2}}\sum_{(x,y,a_1)\in G_1}\sum_{(x',y,a_2)\in G_2} \ell(W f_\text{enc}(x), y))
        - \frac{1}{n_{g_1}}\frac{1}{n_{g_2}}\sum_{(x,y,a_1)\in G_1}\sum_{(x',y,a_2)\in G_2} \ell(W f_\text{enc}(x'), y))\right|. \nonumber
    \end{align*}
    Since we have also assumed that the loss function $\ell(x, y)$ is $C_1$-Lipschitz in $x$\footnote{In other words, we assume that $\abs{\ell(z, y) - \ell(z', y)} \le C_1 \cdot\norm{z - z'}_2$, for any $z, z'$ and $y$.}, the above is at most:
    \begin{align*}
        & \bigabs{\frac{1}{n_{g_1} n_{g_2}} \sum_{(x, y, a_1)\in G_1}\sum_{(x', y, a_2)\in G_2} \bigabs{\ell(W f_{\enc}(x), y) - \ell(W f_{\enc}(x'), y)}} \\
        \le& {\frac{1}{n_{g_1}n_{g_2}} \sum_{(x,y,a_1)\in G_1}\sum_{(x',y,a_2)\in G_2} C_1\cdot \norm{W f_\text{enc}(x) - W f_\text{enc}(x')}_2} \tag{since $y$ is the same for $x, x'$}\\
        \le& \frac{B}{n_{g_1} n_{g_2}} \sum_{(x,y,a_1)\in G_1}\sum_{(x',y,a_2)\in G_2} C_1\cdot \norm{f_\text{enc}(x) - f_\text{enc}(x')}_2 \tag{because $\norm{W}_2 \le B$ as assumed} \\
        \le& B \cdot C_1 \cdot \hat{\cL}_{\alig}(f_{\theta}; y). \tag{because of equation \eqref{eq_group_align}}
    \end{align*}
    Thus, we have shown that for any $g_1$ and $g_2$ within class $y$,
    \begin{align}
        \Delta(g_1, g_2) &\le B \cdot \hat{\cL}_{\alig}(f_{\theta}; y) + \bigbrace{\sqrt{\frac{2\log(\abs{\cG_y}/\delta)}{n_{g_1}}} + \sqrt{\frac{2\log(\abs{\cG_y}/\delta)}{n_{g_2}}}} \nonumber \\
        &\le B \cdot C_1 \cdot \hat{\cL}_{\alig}(f_{\theta}; y) + \max_{g \in \cG_y} C_2 \cdot \sqrt{\frac{8\log(\abs{\cG_y} / \delta)}{n_g} }. \label{eq_delta_bd}
    \end{align}
    Finally, we use the above result to bound the gap between the worst-group loss and the average loss.
    For every group $g\in\cG$, let $p_g$ denote the prior probability of observing a sample from $\cP$ in this group.
    Let $q_y = \sum_{g'\in\cG_y} p_{g'}$.
    Let $h(g)$ be a short hand notation for
    \begin{align*}
        h(g) = \exarg{(x, y, a)\sim\cP_g}{\cL_{\avg}(W f_{\enc}(x), y)}.
    \end{align*}
    The average loss among the groups with class label $y$ is
    $\cL_{\avg}(f_{\theta}; y) = \sum_{g\in\cG_y} \frac{p_g}{q_y} h(g)$.
    The worst-group loss among the groups with class label $y$ is $\cL_{\wg}(f_{\theta}; y) = \max_{g\in\cG_y} h(g)$.
    Let $g^{\star}$ be a group that incurs the highest loss among groups in $\cG_y$.
    We have $\cL_{\wg}(f_{\theta}; y) - \cL_{\avg}(f_{\theta}; y)$ is equal to
    \begin{align}
         h(g^{\star}) - \sum_{g \in \cG_y} \frac{p_g}{q_y} h(g)
        = & \sum_{g\in\cG_y} \frac{p_g}{q_y} (h(g^{\star}) - h(g)) \\
        \le & \sum_{g\in\cG_y} \frac{p_g}{q_y} \Delta(g^{\star}, g) \\
        \le & B\cdot C_1\cdot\hat{\cL}_{\alig}(f_{\theta}; y) + \max_{g\in\cG_y} C_2\cdot \sqrt{\frac{8\log(\abs{\cG}/\delta)}{n_g}}.  \label{eq_thm1_final_eq}
    \end{align}
    The last step uses equation \eqref{eq_delta_bd} on $\Delta(g^{\star}, g)$ and the fact that $q_y = \sum_{g' \in \cG_y} p_{g'}$.
    Thus, we have shown that the gap between the worst-group loss and the average loss among the groups with the same class label is bounded by the above equation.
    The proof is now complete.
\end{proof}

The astute reader will note that Theorem \ref{prop_close} focuses on comparing groups within the same class $y$, for any $y\in\cY$.
A natural follow-up question is what happens when comparing across groups with different labels.
Let $\cL_{\wg}(f_{\theta}) = \max_{y\in\cY} \cL_{\wg}(f_{\theta}; y)$ be the worst-group loss across all the labels.
Recall that $\cL_{\avg}(f_{\theta})$ is the average loss for the entire population of data.
We generalize Theorem \ref{prop_close} to this setting in the following result.

\begin{corollary}[Extension of Theorem \ref{prop_close} to compare across different classes]\label{cor_align}
    In the setting of Theorem \ref{prop_close}, let $q_y = \sum_{g\in\cG_y} p_{g}$ be the prior probability of observing a sample drawn from $\cP$ with label $y$, for any $y \in\cY$.
    We have that with probability at least $1 - \delta$, the following holds:
    \begin{align}
        \cL_{\wg}(f_{\theta}) \le \Big(\min_{y\in\cY} q_y \Big)^{-1} \cL_{\avg}(f_{\theta}) + B\cdot C_1\cdot \max_{y\in\cY} \hat{\cL}_{\alig}(f_{\theta}; y) + \max_{g\in\cG} C_2\cdot\sqrt{\frac{8\log(\abs{\cG} / \delta)}{n_g}}. \label{eq_cor_gen}
    \end{align}
\end{corollary}

\begin{proof}
    We generalize the argument in the previous result to compare across different labels.
    The worst-group loss across different labels is
    {\small\begin{align*}
            & \max_{y\in\cY} \max_{g\in\cG_y} h(g) \\
        \le & \max_{y\in\cY} \bigbrace{\sum_{g\in\cG_y} \frac{p_g}{q_y} h(g) + B\cdot C_1 \hat{\cL}_{\alig}(f_{\theta}; y) + \max_{g\in\cG_y} C_2 \sqrt{\frac{8 \log(\abs{\cG_y} / \delta)}{n_g}}} \tag{because of equation \eqref{eq_thm1_final_eq}} \\
        \le & \frac{1}{\min_{y \in \cY} q_y} \sum_{g\in\cG_y} p_g h(g) + B\cdot C_1 \max_{y\in\cY} \hat{\cL}_{\alig}(f_{\theta}; y) + \max_{g\in\cG} C_2 \sqrt{\frac{8 \log(\abs{\cG} / \delta)}{n_g}}. 
    \end{align*}}
    Since $\sum_{g\in\cG} p_g h(g) = \cL_{\avg}(f_{\theta})$, we thus conclude that
    \begin{align*}
        \mathcal{L}_{\text{wg}}(f_{\theta}) \le \Big(\min_{y\in\mathcal{Y}} q_y\Big)^{-1} \mathcal{L}_{\text{avg}}(f_{\theta}) + B \cdot C_1\max_{y\in\cY} \hat{\cL}_{\alig}(f_{\theta}; y) + \max_{g\in\mathcal{G}} C_2\sqrt{\frac{8\log(\abs{\cG} /\delta)}{n_g}}.
    \end{align*}
    The proof is now complete.
\end{proof}

\noindent\textbf{An example showing that Corollary \ref{cor_align} is tight.} We describe a simple example in which the factor $\Big(\min_{y\in\cY} q_y \Big)^{-1}$ in equation \eqref{eq_cor_gen} is tight (asymptotically).
Suppose there are $k$ perfectly balanced classes so that $q_y = 1 / k$, for every $y \in \cY$. There is one data point from each class, with loss equal to 0 for all except one of them. The worst-group loss is 1 whereas the average loss is $1/k$.
Thus, there is a factor of $k$ between the worst-group loss and the average loss.
For equation \eqref{eq_cor_gen}, the factor \[ \Big(\min_{y\in\mathcal{Y}} q_y\Big)^{-1} = k,\]
since $q_y = 1/k$ for every $y\in\cY$ in this example.
Thus, this factor matches the (multiplicative) factor between the worst-group loss and the average loss in this example.